\documentclass[onecolumn,12pt,draftcls]{IEEEtran}  % Comment this line out
\IEEEoverridecommandlockouts % This command is only
\usepackage{url}
\usepackage{bm}
\usepackage{srcltx}
\usepackage{algorithm}
\usepackage{algorithmic}
\usepackage[utf8x]{inputenc}
\RequirePackage{lineno}
\usepackage{longtable}% for long tables
\usepackage{booktabs}
\usepackage[font=small,labelfont=bf]{caption}
\usepackage{siunitx}

\usepackage{amssymb,latexsym,a msfonts,amsmath,amscd} 
\usepackage{bbold}
\usepackage{subfigure}
\usepackage{cite}    
\usepackage{float}
\usepackage{verbatim}
\usepackage{acronym}
\usepackage{paralist}
\usepackage{graphicx}
\usepackage{gastex}
\usepackage{acronym}
\usepackage{booktabs}

\graphicspath{{./figure/}}
\DeclareGraphicsExtensions{.eps}

\let\oldmarginpar\marginpar
\renewcommand\marginpar[1]{\-\oldmarginpar[\raggedleft\footnotesize #1]%
{\raggedright\footnotesize #1}}
\usepackage[all]{xy}

\newcommand{\nat}{\mathbb{N}}

\makeatletter
\def\imod#1{\allowbreak\mkern5mu({\operator@font mod}\,\,#1)}
\makeatother

\renewcommand{\emph}{\textit}

\newtheorem{theorem}{Theorem}
\newtheorem{lemma}{Lemma}

\newtheorem{proposition}{Proposition}

\newtheorem{definition}{Definition}

 \newtheorem{assumption}{Assumption}

 \title{Symbolic Planning and Control Using Game Theory and
   Grammatical Inference}

 \author{Jie~Fu,~\IEEEmembership{Student~Member,~IEEE,} Herbert~
   G.~Tanner,~\IEEEmembership{Senior~Member,~IEEE,} Jeffrey~Heinz,
   Jane~Chandlee,
   Konstantinos~Karydis,~\IEEEmembership{Student~Member,~IEEE}, and
   Cesar~Koirala \thanks{Jie Fu, Herbert G. Tanner and Konstantinos
     Karydis are with the Mechanical Engineering Department at the
     University of Delaware, Newark DE 19716.
     \texttt{\{jiefu,kkaryd,btanner\}@udel.edu}.}  \thanks{Jeffrey
     Heinz, Jane Chandlee and Cesar Koirala are with the Department of
     Linguistics and Cognitive Science at the University of Delaware,
     Newark DE 19716. \texttt{\{heinz,janemc,koirala\}@udel.edu}.}
   \thanks{This work is supported by NSF award \#1035577.
   The authors thank Calin Belta and his group for joint technical
discussions through which the case study game example was conceived.
Thanks are also extended to Jim Rogers for his insightful comments.
}}

\usepackage[colorlinks=true]{hyperref}

\begin{document}

\acrodef{dfa}[DFA]{deterministic finite state automaton}
\acrodef{fsas}[FSAs]{finite state automata} \acrodef{fsa}[FSA]{finite
  state automaton} \acrodef{sa}[SA]{semiautomaton}
\acrodef{tsl}[TSL]{Tier-based Strictly Local}
\acrodef{mso}[MSO]{second-order monadic logic of words}
\acrodef{re}[regex]{Regular expression} \acrodef{fsa}[FSA]{finite
  state automaton} 
  \acrodef{gim}[GIM]{grammatical inference machine}
\acrodef{sel}[\textsc{SEL}]{string extension learning}
\acrodef{sl}[SL]{strictly local} \acrodef{sp}[SP]{strictly piecewise}
\acrodef{pt}[PT]{prefix tree}
\acrodef{sef}[SEF]{string extension function}
\acrodef{ltl}[LTL]{Linear Temporal Logic}

%\makeatletter
\newcommand{\rmnum}[1]{\romannumeral #1}
\newcommand{\Rmnum}[1]{\expandafter\@slowromancap\romannumeral #1@}
\makeatother

%\setpagewiselinenumbers
%\modulolinenumbers[5]
%\linenumbers
\maketitle

\begin{abstract}
  This paper presents an approach that brings together game theory
  with grammatical inference and discrete abstractions in order to
  synthesize control strategies for hybrid dynamical systems
  performing tasks in partially unknown but rule-governed adversarial
  environments.  The combined formulation guarantees that a system
  specification is met if
\begin{inparaenum}[(a)]
\item the true model of the environment is in the class of models
  inferable from a positive presentation,  \item a
  characteristic sample is observed,   and \item the task
    specification is satisfiable given the capabilities of the system
    (agent) and the environment.
\end{inparaenum} 
%Specifically, the agent begins by assuming a static environment and then
%iteratively constructs an increasingly more accurate model of the
%environment and its adversary's capacity.  It does this with a string
%extension learner applied to the prefixes of the strictly local language
%accepted by its adversary discrete  model.  
%It is proven that for certain configurations
%and starting positions of the example game, the agent is guaranteed to
%converge to the true model and can devise a winning strategy.
\end{abstract}

\section{Introduction}
\subsection{Overview}

This paper demonstrates how a particular method of machine learning can be
incorporated into hybrid system planning and control, to enable systems to accomplish complex
tasks in \emph{unknown} and \emph{adversarial} environments.
This is achieved by bringing together formal abstraction methods for
hybrid systems, grammatical inference and (infinite) game theory.

%---------------------------------
 Many, particularly commercially available, automation systems 
come with control user interfaces that involve continuous low-to-mid level
controllers, which are either specialized for the particular application,
or are designed with certain ease-of-use, safety, or performance specifications in mind.  
%Examples of the first category include legged locomotion systems \cite{pougi11}
%where motion controllers are of very special form, and 
%medical and rehabilitation devices \cite{sunil-impedance,davinci} where
%operation needs to adhere to very specific safety standards.
%Instances of the latter category are cases where fiddling with the low-level
%control loops can have significant safety, and economic or environmental 
%consequences \cite{io-con,auto}.  Limiting
%access to low-level control loops in commercial automated systems 
%may also be motivated by marketing reasons \cite{davinci}.
This paper proposes a control synthesis method that works 
with---rather than in lieu of---existing control loops.  The focus here is on how
to abstract the given low-level control loops \cite{Tanner2012}
and the environment they operate in \cite{Fainekos2009},
and combine simple closed loop behaviors in an orchestrated temporal sequence.
The goal is to do so in a way that guarantees
the satisfaction of a task specification and is provably implementable
at the level of these low-level control and actuation loops.

As a field of study, grammatical inference is primarily concerned with
developing algorithms that are guaranteed to learn how to identify 
any member of a collection of formal objects  (such as languages or graphs) from a
presentation of examples and/or non-examples of that object, provided
certain conditions are met \cite{delaHiguera2010}. The conditions are
typical in learning research: the data presentation must be adequate,
the objects in the class must be reachable by the generalizations
the algorithms make, and there is often a trade-off between the two.

Here, grammatical inference is integrated into planning and control
synthesis using game theory.  Game theory is a natural framework for
reactive planning of a system in a dynamic environment
\cite{Zielonka1998135}. A task specification becomes a winning
condition, and the controller takes the form of a strategy that
indicates which actions the system (player 1) needs to take so that
the specification is met regardless of what happens in its environment
(player 2) \cite{Ramadge1987,game-Julia97}.  It turns out that
interesting motion planning and control problems can be formulated at
a discrete level as a variant of reachability games \cite{Gradel2002},
in which a \emph{memoryless} winning strategy can be computed for one
of the players, given the initial setting of the game.

In the formulation we consider, the rules of the game are assumed to
be initially unknown to the system; the latter is supposed
to operate in a potentially adversarial environment with unknown
dynamics. The application of grammatical inference algorithms to the
observations collected by the system during the course of the game
enables it to construct and incrementally update a model of this
environment. Once the system has learned the true nature of the game,
and if it is possible for it to win in this game, then it \emph{will}
indeed find a winning strategy, no matter how effectively the
adversarial environment might try to prevent it from doing so. In
other words, the proposed framework guarantees the satisfaction of the
task specification in the face of uncertainty, provided certain
conditions are met.  If those conditions are not met, then the system
is no worse off than when not using grammatical inference algorithms.

\subsection{Related work}

So far, symbolic planning and control methods address problems
where the environment is either static and presumably known, or
satisfies given assumptions \cite{Piterman06synthesisof,Belta2007,Lahijanian2010}.

In cases where the environment is static and known, we see
applications of formal methods like model checking
\cite{Belta2007,Laviers11}.  In other variants of this formulations,
reactive control synthesis is used to tackle cases where system
behavior needs to be re-planned based on information obtained from the
environment in real time \cite{Piterman06synthesisof}.  In
\cite{Lahijanian2010} a control strategy is synthesized for maximizing
the probability of completing the goal given actuation errors and
noisy measurements from the environment.  Methods for ensuring that
the system exhibits correct behavior even when there is the mismatch
between the actual environment and its assumed model are proposed in
\cite{Wongpiromsarn2010}.

\ac{ltl} plays an important role in existing approaches to 
symbolic planning and control.   It is being used to capture
\emph{safety}, \emph{liveness} and \emph{reachability} specifications
\cite{Tomlin2003}.  A formulation of \ac{ltl} games on graphs is
used in \cite{Kloetzer2008} to synthesize
control strategies for non-deterministic transition systems. 
Assuming an uncertain system model,
\cite{Wongpiromsarn2010} combines temporal logic
control synthesis with receding horizon control concepts.  
Centralized control designs for groups of robots tasked with
satisfying a \ac{ltl}-formula specification are found in \cite{Hadas2009}, 
under the assumption that the environment in which the robots
operate in adheres to certain conditions.  These methods are
extended \cite{Hadas2011} to enable the plan to
be revised during execution.  

Outside of the hybrid system's area, adjusting unknown system
parameters has traditionally been done by employing 
adaptive control or machine learning methods.
Established adaptive control techniques operate in a purely continuous
state regime, and most impose stringent conditions (e.g., linearity) on the
system dynamics; for these reasons they are not covered in the context of this limited scope
review---the interested reader is referred to \cite{Astrom,sastry-adaptive}.
On the other hand, machine learning is arguably a broader field.
A significant portion of existing
work is based on \emph{reinforcement learning}, which has
been applied to a variety of problems such as multi-agent control
\cite{maja97}, humanoid robots \cite{schaal03}, varying-terrain
wheeled robot navigation \cite{roy09}, and unmanned aerial vehicle
control \cite{abbeel10}.  The use of grammatical inference as a 
sub-field of machine learning in the context of robotics and control
is not entirely new; an example is the application of a 
\ac{gim} in robotic self-assembly \cite{Hamdi-Cherif2009}.

In the aforementioned formulations there is no consideration for
dynamic adversarial environments.  A notable exception is the work of
\cite{Yushan2012}, which is developed in parallel to, and in part
independently from, the one in this paper.  The idea of combining
learning with hybrid system control synthesis is a natural common
theme since both methods originate from the same joint sponsored
research project.  Yet, the two approaches are distinct in how they
highlight different aspects of the problem of synthesis in the
presence of dynamic uncertainty.  In \cite{Yushan2012}, the learning
module generates a model for a stochastic environment in the form of a
Markov Decision Process and control synthesis is performed using model
checking tools.  In this paper, the environment is deterministic, but
intelligently adversarial and with full knowledge of the system's
capabilities. In addition, the control synthesis here utilizes tools
from the theory of games on infinite words.

\subsection{Approach and contributions}

This paper introduces a symbolic control synthesis method based on the
architecture of Fig.~\ref{fig:hybrid-learning}, where a \ac{gim} is
incorporated into planning and control algorithms of a hybrid system
(a robot, in Fig.~\ref{fig:hybrid-learning}) to identify the dynamics
of an evolving but rule-governed environment.  The system---its
boundaries outlined with a thick line---interacts with its environment
through sensors and actuators.  Both the system as well as its
environment are dynamical systems (shown as ovals), assumed to admit
discrete abstractions in the form of transition systems (dashed
rectangles).  The system is required to meet a certain specification.
Given its specification ($\mathcal{A}_s$), an abstraction of itself ($A_1$), and
its hypothesis of the dynamics of its environment ($A_2$), the system
devises a plan and implements it utilizing a finite set of low-level
concrete control loops involving sensory feedback.  Using this sensory
information, the system refines its discrete environment model based
on a \ac{gim}, which is guaranteed to identify the environment
dynamics asymptotically.   Figure~\ref{fig:learning} gives a general description of the
implementation of learning and symbolic planning at the high-level of
the architecture in Fig.~\ref{fig:hybrid-learning}. The hypothesis on
the environment dynamics is at the center of the system's planning
algorithm. Through interactions with the environment, the system
observes the discrete evolution $\phi(i)$ of the environment dynamics,
and uses the \ac{gim} to construct and update a hypothesized
environment model $A_2^{(i)}$. Based on the environment model, the
system constructs a hypothesis (model) $\mathcal{G}^{(i)}$ capturing
how the ``game'' between itself and the environment is played, and
uses this model to devise a winning strategy (control law)
$\mathsf{WS}_1^\ast$. As the environment model converges
asymptotically to the true dynamics $A_2$, the winning strategy
becomes increasingly more effective.  In the limit, the system is
guaranteed to win the game.

\begin{figure}[t]
  \centering \subfigure[An overview of the architecture]{
  \label{fig:hybrid-learning}
\includegraphics[width=3in]{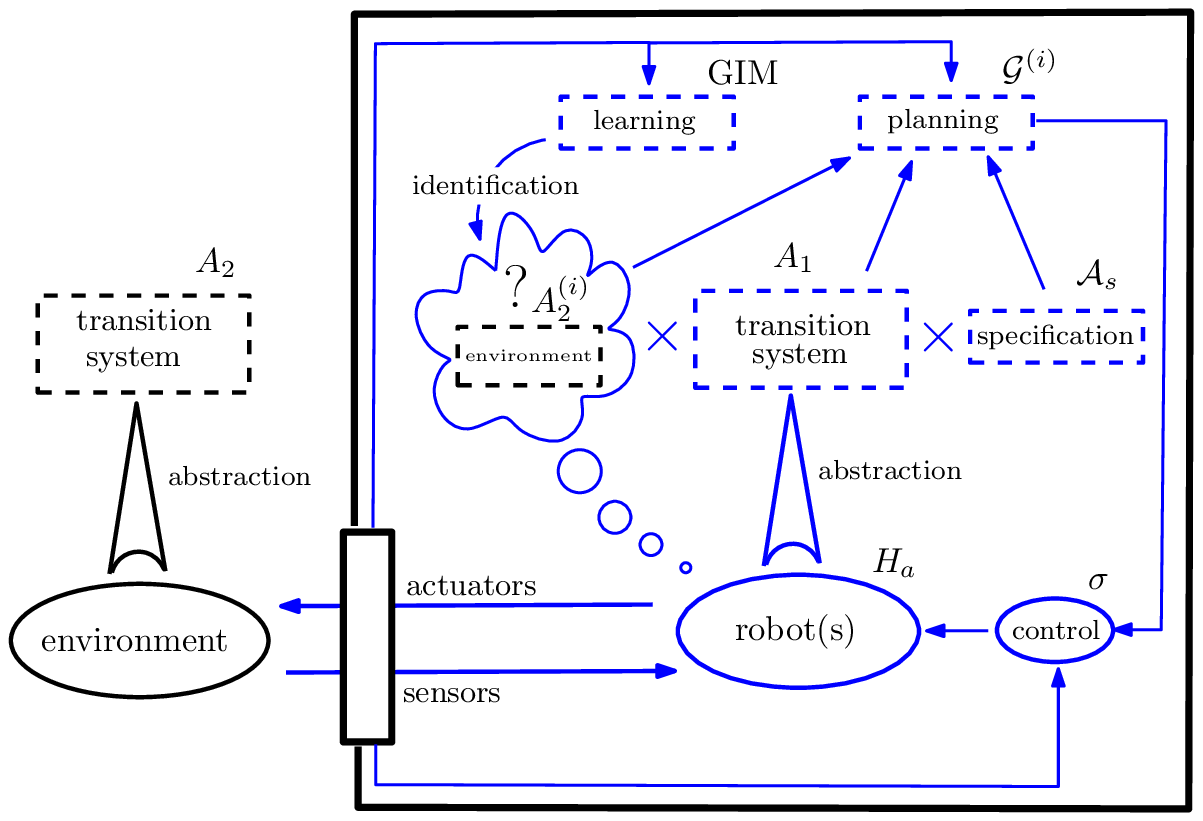}}
\subfigure[Learning and planning with grammatical inference module
      at the higher level.]{\label{fig:learning}
    \includegraphics[width=3in]{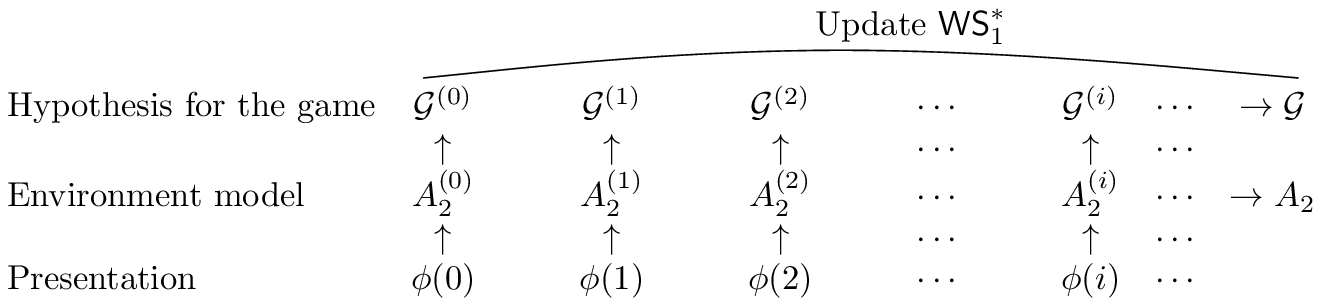}
}
\caption{The architecture of hybrid agentic planning and control with a
    module for grammatical inference.}
\end{figure}
%\vspace*{-5ex}

Definitions~\ref{def:turnbased-prod}, \ref{def:gamefsa} and
Theorem~\ref{thm:simulation} establish how a game can be constructed
from the system abstractions of the (hybrid) system dynamics ($A_1$),
the environmental dynamics ($A_2$), and the task specification
($\mathcal{A}_s$). Theorem~\ref{thm:win} proves that the hybrid agent
can determine whether a winning strategy exists, and if it does, what
it is.  Grammatical inference methods yield
increasingly accurate models of environmental dynamics (assuming
adequate data presentations and reachable targets), and permit the
system to converge to an accurate model of its environment. Discrete
backward reachability calculations can be executed in a
straightforward manner and can allow the determination of winning
strategies (symbolic control laws), whenever the latter exist.

The contribution of this paper is two-fold:
\begin{inparaenum}[(i)]
\item it integrates \ac{gim}s into hybrid systems for the purpose of
  identifying the discrete dynamics of the environment that evolve and
  possibly interact with the system, and
\item it uses the theory of games on infinite words for symbolic control synthesis,
and discrete abstractions which ensure implementation of the symbolic
plans on the concrete hybrid system.
\end{inparaenum}
In the paper, both elements are combined, but each element has merit
even in isolation.  A hybrid system equipped with \ac{gim} is still
compatible with existing symbolic control synthesis methods (including
model checking).  On the other hand, the abstractions methods we
utilize here---although requiring strong properties on the continuous
components dynamics of the hybrid system---offer discrete abstract
models which are weakly simulated by the concrete systems,
irrespectively of whether the latter include a \ac{gim} or not.

% Figure~\ref{fig:hybrid-learning} illustrates the architecture of our
% solution. It shows a hybrid agentic system including a module for
% grammatical inference.

\subsection{Organization}
 The rest of this paper is organized as follows.  Section
\ref{section:preliminaries} introduces the technical background,
the notation, and the models used.  The type of hybrid systems
considered and their discrete abstractions are presented there.
In Section \ref{section:analysis}, we show how the control
problem can be formulated as a game and employ the concept of
the attractor in games 
for control synthesis. Section \ref{section:GI} describes first how
a \ac{gim} can be used to identify asymptotically the dynamics of
the system's unknown and adversarial environment, and then how
this knowledge can be utilized in planning and control synthesis. 
In Section \ref{refine}, we establish the properties of the
relation between the hybrid system and its discrete abstraction, 
which ensure that the strategy devised based on the discrete model
is implementable on the concrete system.  Section 
\ref{section:example} illustrates the whole approach through an
example robotic application. In Section
\ref{section:conclusion} we discuss possible extensions of the 
proposed methodology and compare our grammatical inference  
to other learning methods. 
%Section \ref{section:conclusion} concludes the
%paper.

%=============================================
\section{Technical Preliminaries}
\label{section:preliminaries}

\subsection{Languages and Grammatical Inference}
\label{section:grinf-pre}

Let $\Sigma$ denote a fixed, finite alphabet, and $\Sigma^n$,
$\Sigma^{\leq n}$, $\Sigma^*$, $\Sigma^\omega$ be sequences over this
alphabet of length $n$, of length less than or equal to $n$, of any
finite length, and of infinite length, respectively.  The \emph{empty
  string} is denoted $\lambda$, and the \emph{length of string} $w$ is
denoted $|w|$.  A \emph{language} $L$ is a subset of $\Sigma^*$. A
string $u$ is a prefix (suffix) of a string $w$ if and only if there
exists a string $v$ such that $w=uv$ ($w=vu$). A prefix (suffix) of
length $k$ of a string $w$ is denoted $\mathsf{Pr}^{=k}(w)$
\big(respectively, $\mathsf{Sf}^{=k}(w)$\big) and a set of prefixes
(suffixes) of a string $w$ of length $\le k$ is denoted as
$\mathsf{Pr}^{\le k}(w)$ \big(respectively, $\mathsf{Sf}^{\le
  k}(w)$\big). For $w=\sigma_1\sigma_2\cdots\sigma_n\in\Sigma^*$, the
\emph{shuffle ideal} of $w$ is defined as
$\mathsf{SI}(w):=\Sigma^*\sigma_1\Sigma^*\sigma_2\cdots
\Sigma^*\sigma_n\Sigma^*$.  A string $u$ is a \emph{factor} of string
$w$ iff $\exists x,y\in\Sigma^*$ such that $w=xuy$. If in addition
$|u|=k$, then $u$ is a \emph{$k$-factor} of $w$.  If $E$ is a set,
$2^E$ denotes the set of all subsets and $2_\mathsf{fin}^{E}$ the set
of all finite subsets of $E$. A \ac{sef} is a total function,
$\mathfrak{f} : \Sigma^\ast \rightarrow 2_\mathsf{fin}^E$.  The
\emph{k-factor} function $\mathfrak{f}_k:\Sigma^*\rightarrow
2_\mathsf{fin}^{\Sigma^{\leq k}}$ maps a word to the set of
$k$-factors within it. If $|w|\le k$, $\mathfrak{f}_k(w):=\{w\}$,
otherwise $\mathfrak{f}_k(w):=\{u \mid u\mbox{ is a }k\mbox{-factor of
}w\}$. This function is extended to languages as
$\mathfrak{f}_k(L):=\bigcup_{w\in L} \mathfrak{f}_k(w)$.

A \ac{sa} is a tuple $A= \langle Q,\Sigma,T \rangle$ where $Q$ is the
set of states, $\Sigma$ is the set of alphabet and the transition
function is $ T: Q\times \Sigma \rightarrow Q$.  The elements of
$\Sigma$ are referred to as \emph{actions} and are thought to initiate
transitions at a given state according to $T$.  If $T(q_1,\sigma)=q_2$
(also written as $q_1\stackrel{\sigma}{\rightarrow} q_2$) with
$\sigma \in \Sigma$, then we say that $A$ \emph{takes action} $\sigma$
on $q_1$ and moves to $q_2$.  The transition function is expanded
recursively in the usual way. Note by definition, these \ac{sa}s are
deterministic in transition. For a (semi)automaton $A$, we define the
set-valued function $\Gamma:Q\rightarrow 2^\Sigma$ as
$\Gamma(q):=\{\sigma\in\Sigma\mid T(q,\sigma) \mbox{ is defined}\}$.
% For a \ac{sa} $A$, $\mathcal{L}(A) :=\{w \mid T(q,w) \mbox{ is
%   defined}\}$ is the language \emph{admissible} in $A$.
A \ac{fsa} is a tuple $\mathcal{A}= \langle A, I, F \rangle$ where
$A=\langle Q,\Sigma,T \rangle$ is a semiautomaton and $I,F\subseteq Q$
are the initial and final states, respectively.  The language of a
\ac{fsa} is $L(\mathcal{A}):=\{w\mid T(I,w)\cap F\neq\emptyset
\}$. For a regular language $L$, deterministic \ac{fsa}s recognizing
$L$ with the fewest states are called
\emph{canonical}.

For concreteness, let grammars of languages be constructed as the set
of possible Turing machines $\mathfrak{G}$. (Other kinds of grammars
are used later, but they are translatable into Turing machines.) The
language of a particular grammar $\mathfrak{G}$ is
$L(\mathfrak{G})$. A \emph{positive presentation} $\phi$ of a language
$L$ is a total function $\phi :\mathbb{N}\rightarrow L \cup\{\#\}$
($\#$ is a `pause'\footnote{Pause $\#$ can be understood as ``non
  data.''}) such that for every $w\in L$, there exists $n \in
\mathbb{N}$ such that $\phi(n)=w$.  With a small abuse of notation, a
presentation $\phi$ can also be understood as an infinite sequence
$\phi(1)\phi(2)\cdots$ containing every element of $L$.

Let $\phi[i]$ denote the initial finite sequence
$\phi(1)\phi(2)\ldots\phi(i)$.  Let $\mathfrak{Seq}$ denote the set of
all finitely long initial portions of all possible presentations of
all possible languages (i.e., all $\phi[i]$ for all $i\in\mathbb{N}$
and for all $L$). The \emph{content} of $\phi[i]$, written
$\textsf{content}(\phi[i])$, is the set of the elements of the
sequence, less the pauses. A \emph{learner} (\emph{learning algorithm,
  or \ac{gim}}) is a program that takes the first $i$ elements of a
presentation and returns a grammar as output:
$\mathfrak{Gim}:\mathfrak{Seq}\to \mathfrak{G}$. The grammar returned
by $\mathfrak{Gim}$ is the learner's \emph{hypothesis} of the
language.  A learner $\mathfrak{Gim}$ \emph{identifies in the limit
  from positive presentations} of a collection of languages
$\mathcal{L}$ if and only if for all $L\in\mathcal{L}$, for all
presentations $\phi$ of $L$, there exists a $n\in \mathbb{N}$ such
that for all $m\ge n$, $\mathfrak{Gim}(\phi_m) = \mathfrak{G}$ and
$L(\mathfrak{G})=L$ \cite{gold67}.  A \emph{characteristic sample} $S$
for a language $L$ and a learner $\mathfrak{Gim}$ is a finite set of
strings belonging to $L$ such that for any $\phi[i]$ such that
$\mathsf{content}(\phi[i])=S$, it is the case that for all $j\ge i$,
$\mathfrak{Gim}(\phi[j])=\mathfrak{G}$ and $L(\mathfrak{G})=L$.

\begin{definition}[String extension grammar and languages \cite{Heinz-2010-SEL}]
  Let $\mathfrak{f}$ be a \ac{sef}, and $E$ be a set.  A \emph{string
    extension grammar} $\mathfrak{G}$ is a finite subset of $E$.  The
  \emph{string extension language of grammar} $\mathfrak{G}$ is
  $L_{\mathfrak{f}}(\mathfrak{G})=\{w\in \Sigma^\ast:
  \mathfrak{f}(w)\subseteq \mathfrak{G}\}.$ The \emph{class of string
    extension languages} is $ \mathcal{L}_{\mathfrak{f}}
  :=\{L_{\mathfrak{f}}(\mathfrak{G}):\mathfrak{ G}\in
  2_\mathsf{fin}^E\}.$
\end{definition}

\begin{definition}[String Extension Learner\cite{Heinz-2010-SEL}]
Let $\mathfrak{f}$ be a \ac{sef}.  For all positive presentations
$\phi$, define $\mathfrak{Gim}_{\mathfrak{f}}$ as: $ \mathfrak{Gim}_{\mathfrak{f}}(\phi[i]) = \emptyset$
if $i=0$, and
\begin{equation}
  \mathfrak{Gim}_{\mathfrak{f}}(\phi[i]) :=\begin{cases}
      \mathfrak{Gim}_{\mathfrak{f}}(\phi[i-1])                & \text{ if } \phi(i)=\# \\
      \mathfrak{Gim}_{\mathfrak{f}}(\phi[i-1])\cup \mathfrak{f}(\phi[i]) & \text{ otherwise}\enspace.
\end{cases}
\end{equation}
\end{definition}

According to \cite{gold67}, the class of regular languages is not
identifiable in the limit from positive presentation, but string
extension languages---which are subclasses of regular languages---are.
\begin{theorem}[\!\!\cite{Heinz-2010-SEL}]
Learner $\mathfrak{Gim}_{\mathfrak{f}}$ identifies $\mathcal{L}_{\mathfrak{f}}$ in the limit.
\end{theorem}
Many attractive properties of
string extension learners are established in \cite{Koetzing2010}.
A language $L$ is Strictly $k$-Local ($\text{\ac{sl}}_k$)
\cite{McNaughtonPapert1971,Luca1980} iff there exists a finite set
$S\subseteq \mathfrak{f}_k(\rtimes \Sigma^*\ltimes)$, such that $L=\{
w\in\Sigma^* : \mathfrak{f}_k(\rtimes w\ltimes)\subseteq S\}$, where
$\rtimes, \ltimes$ are the symbols indicating the beginning and end of
a string, respectively.  Obviously, Strictly $k$-Local languages are
string extension languages.  The following theorem follows
immediately.
\begin{theorem}[\!\!\cite{GarciaEtAl1990}]
\label{thm:sl}
  For every $k$,  Strictly $k$-Local languages are identifiable in
  the limit from positive presentations.
\end{theorem}

\begin{theorem}[\ac{sl}-Hierarchy \!\!\cite{RogersTalk}]
\label{thm:hierarchy}
$\text{ \ac{sl}}_1\subset \text{\ac{sl}}_2 \subset \ldots \subset
\text{\ac{sl}}_i\subset \text{\ac{sl}}_{i+1}\subset \ldots
\text{\ac{sl}}$.
\end{theorem}
The implication of Theorem \ref{thm:hierarchy} is that any Strictly
$k$-Local language can be described using a \ac{sl}$_j$ grammar, where
$j\ge k$.  Section \ref{section:GI} illustrates this
argument with the help of an example.

\subsection{Hybrid Systems and Abstractions}
\label{sec:hds}

A \emph{hybrid system} $H$ is defined as a tuple of objects 
(for a precise definition, see \cite{LygerosTAC}) that includes
the domains of continuous and discrete variables, the subsets of
initial states in those domains, the description of the family of
continuous dynamics parametrized by the discrete states, and rules
for resetting continuous and discrete states and switching between the
members of the family of continuous dynamics.
%\begin{definition}[Hybrid automaton \cite{LygerosTAC}]
%A hybrid automaton $H$ is a collection 
%\linebreak
%$H = \langle Q, \mathcal{X}, f, \mathrm{Init}, D, E, G, R\rangle$: where
%\begin{inparaenum}[(i)]
%\item $Q$ is a finite set of discrete variables,
%\item $\mathcal{X}$ is a finite dimensional continuous domain,
%\item $f: Q \times \mathcal{X} \to T\mathcal{X}$ is a vector field,
%\item $\mathrm{Init} \subseteq Q \times X$ is a set of initial states,
%\item $D: Q \to 2^\mathcal{X}$ is an invariant set in which continuous variables 
%evolve for given $q\in Q$,
%\item $E \subset Q \times Q$ are transitions between discrete variables,
%\item $G: E \to 2^\mathcal{X}$ are guard conditions that trigger transitions,
%\item $R: E \times \mathcal{X} \to 2^\mathcal{X}$ is a map resetting the continuous variables
%after a transition.
%\end{inparaenum}
%\end{definition}

In this paper, we restrict our attention to a specific class of hybrid
systems where the continuous dynamics have specific (set) attractors
\cite{Tanner2012}.  The shape and location of these attractors are assumed
dependent on a finite set of continuous parameters that are selected
as part of closing the outer control loop.  Judicious selection of the
parameters activates a specific sequence of continuous and discrete
transitions, which in turn steers the hybrid system $H$ from a given initial
state to a final desired state.  This class admits purely
  discrete (predicate-based) abstractions.
We call these particular types of hybrid automata \emph{hybrid agents}, to distinguish
them from general cases.

\begin{definition}[Hybrid Agent]
\label{hybrid-agent}
The hybrid agent is a tuple:

$H_a= \langle { \mathcal{Z}, \Sigma_a, \iota, \mathcal{P}, \pi_i,
  \mathcal{AP}, f_\sigma, \text{\textsc{Pre}} ,\text{\textsc{Post}},
  s, T_a} \rangle$.
\begin{itemize}
\item $\mathcal{Z} = {\cal X} \times {\bm L}$ is a set of
  \emph{composite} (continuous and Boolean) states, where ${\cal X}
  \subset \mathbb{R}^{n}$ is a compact set, and ${\bm L }\subseteq
  \left\{{\bm{0,1}}\right\}^{r}$ where $r$ is the number of Boolean
  states.
  
\item $\Sigma_a$ is a set of finite discrete states (\emph{control modes}).
	
\item $\iota: \Sigma_a \to \{1,\ldots,k\}$ is a function, indexing
        the set of symbols in $\Sigma_a$.% If two control modes
        % $\sigma_1,\sigma_2\in \Sigma$ share the same parameter,
        % $\iota(\sigma_1)=\iota(\sigma_2)$.
	
\item $\mathcal{P} \subseteq \mathbb{R}^m$ is a (column) vector of continuous parameters.

\item $\pi_i: \mathbb{R}^m \to \mathbb{R}^{m_i}$, for 
	$i=1,\ldots,k$ is a finite set of canonical projections, 
	such that $p = ( \pi_1(p)^\mathsf{T}, \ldots, \pi_k(p)^\mathsf{T} )^\mathsf{T}$.

      \item $\cal{AP}$ is a set of (logical) atomic propositions over
        $\mathcal{Z}\times\mathcal{P}$, denoted
        $\{\alpha_h(z,p)\}_{i=1}^{|\mathcal{AP}|}$.  A set of
        well-formed formulae $\mathsf{WFF}$ \cite{enderton} is
        defined inductively as follows: 
        \begin{inparaenum}[(a)]
        \item if $\alpha \in \mathcal{AP}$, then $\alpha \in \mathsf{WFF}$;
        \item if $\alpha_1$ and $\alpha_2$ are in $\mathsf{WFF}$, then
          so are $\neg \alpha_1$ and $\alpha_1 \land \alpha_2$.
        \end{inparaenum}

      \item $f_\sigma$: $\mathcal{Z} \times {\cal{P}} \to T{\cal{X}}$
        is a finite set of families of vector fields parametrized by
        $p \in \mathcal{P}$, $\ell \in \bm L$ and $\sigma \in \Sigma$,
        with respect to which $\mathcal{X}$ is positively
        invariant. These vector fields have limit sets\footnote{The
          compactness and invariance of $\mathcal{X}$ guarantee the
          existence of attractive, compact and invariant limit sets
          \cite{khalil}.} parametrized by $p$ and $\sigma$, denoted
        $L^+(p,\sigma)$.
  
\item $\text{\textsc{Pre}}$: $\Sigma_a \to \mathsf{WFF}$ 
  maps a discrete state to a formula that needs to be satisfied whenever
  $H_a$ switches to discrete state $\sigma$ from any other state.  
  When composite state $z$ and parameter vector $p$ satisfy this formula
  we write $(z,p) \models \text{\textsc{Pre}}(\sigma)$.

\item $\text{\textsc{Post}}$: $\Sigma_a \to \mathsf{WFF}$ maps a
  discrete location to a formula that is satisfied when the
  trajectories of $f_\sigma$ reach an
  $\epsilon$-neighborhood\footnote{Written $L^+(p,\sigma) \oplus
    \mathcal{B}_\varepsilon$, where $\oplus$ denotes the Minkovski
    (set) sum and $\cal B_\varepsilon$ is the open ball of radius
    $\varepsilon$.}  of their limit set.  When composite state $z$ and
  parameter vector $p$ satisfy this formula we write $(z,p) \models
  \text{\textsc{Post}}(\sigma)$.
  
\item $s$: ${\cal Z} \times {\cal P} \rightarrow 2^{\cal P} $ is the reset
  map for the parameters.  It assigns to each pair of composite
  state and parameter a subset of $\cal P$ which contains all 
  values to which the current value of $p \in \cal P$ can be reassigned to.

\item $T_a$: ${\cal Z} \times {\cal P}\times {\Sigma_a} \rightarrow {\cal
  Z} \times {\cal P}\times {\Sigma_a} $ is the discrete state transition map, according
  to which $\left(z,p,\sigma\right)\to \left(z,p',\sigma'\right)$ iff
  $(z,p)\models \text{\textsc{Post}}(\sigma) $ and 
  $(z,p')\models \text{\textsc{Pre}}(\sigma')$ with $p' \in
  s(z,p)$.   
\end{itemize}
The \emph{configuration} of $H_a$ is denoted $h := [z,p,\sigma]$, and
for each discrete state, we define the following subsets of
$\mathcal{Z}\times \mathcal{P}$: $\overleftarrow{\sigma} := \{(z,p) :
(z,p) \models \text{\textsc{Pre}}(\sigma) \}$ and $
\overrightarrow{\sigma} := \{ (z,p) : (z,p) \models
\text{\textsc{Post}}(\sigma) \}$.  A transition from $\sigma_i$ to
$\sigma_{i+1}$ (if any) is forced and occurs at the time instance when
the trajectory of $f_{\sigma_i} \left(x,\ell,p\right)$ hits a nonempty intersection
of a $\varepsilon$-neighborhood of its limit set and the region
of attraction of $\sigma_{i+1}$ parametrized by $p'$
($p'$ not necessarily equals $p$.)  After a transition
$\left(z,p,\sigma\right)\to \left(z,p',\sigma'\right)$ occurs, 
the composite state $z$ \emph{evolves} into 
composite state $z'$ for which $(z',p') \models
\text{\textsc{Post}}(\sigma')$.  The (non-instantaneous) evolution is
denoted $ z \stackrel{\sigma'[p']}{\hookrightarrow} z'$.
\end{definition}

We will use a form of predicate abstraction to obtain a coarse,
discrete representation of $H_a$.  Our abstraction map is denoted $V_M
: \mathcal{Z} \times \mathcal{P} \to
\{\boldsymbol{0},\boldsymbol{1}\}^{|\mathcal{AP}|}$ and referred to as
the \emph{valuation map}:

\begin{definition}[Valuation map]
\label{def:vmap}
The valuation map $V_M$: ${\mathcal{Z} \times \cal P }\to {\cal V}
\subseteq \left\{\bm{1,0}\right\}^{|\cal{AP}|} $ is a function that
maps pairs of composite states and parameters, to a binary vector
$v\in \mathcal{V}$ of dimension $|\mathcal{AP}|$.  The element at
position $i$ in $v$, denoted $v[i]$, is $\bm 1$ or $\bm 0$ if $
\alpha_i \in \mathcal{AP}$ is true or false, respectively, for a
particular pair $(z,p)$.  We write $\alpha_i(z,p)=v[i]$, for $v \in
\cal V$.
\end{definition}

The purely discrete model that we use as an abstraction of $H_a$,
referred to as the \emph{induced transition system} is
defined in terms of the valuation map as follows.

\begin{definition}[Induced transition system] 
\label{def:its}
A hybrid agent $H_a$ induces a semiautomaton $A(H_a)  = \langle
Q,\Sigma, T\rangle$ in which 
\begin{inparaenum}[(i)]\item $Q= V_M(\mathcal{Z}\times \mathcal{P})$
  is a finite set of states; \item $\Sigma = \Sigma_a \cup
  \{\tau_1,\ldots, \tau_m\}$, $m \le |Q\times Q|$ is a finite set of
  labels; \item $T \subseteq Q\times \Sigma \times Q$ is a transition
  relation with the following semantics:
  $q\stackrel{\sigma}{\rightarrow}q' \in T$ iff either\end{inparaenum}
\begin{inparaenum}[(1)]
\item $\sigma \in \Sigma_a$ and
 $\left(\exists p \right)\left(\forall z \in \{z\mid
  V_M(z,p)=q\}\right)$ $\left(\forall z'
  \in \{z'\mid (z',p) \models \text{\textsc{Post}}(\sigma)\} \right)$
$\left[(z,p) \models 
  \text{\textsc{Pre}} (\sigma), V_M(z',p)=q'\right]$, or
\item $\sigma \in \Sigma\setminus \Sigma_a$ and
 $\left(\exists p
  \right)$ $\left(\forall z \in \{z\mid V_M(z,p) =q \}\right)(\exists
  p'\in s(z,p), \sigma'\in \Sigma_a) \left[V_M(z,p') = q',\,
    (z,p')\models \text{\textsc{Pre}}(\sigma') \right]$.
\end{inparaenum}
% \begin{center}
% \begin{tabular}{l l}
% $Q= V_M(\mathcal{Z}\times \mathcal{P})$ & is a finite set of states \\
% $\Sigma = \Sigma_a \cup \{\tau_1,\ldots, \tau_m\}$, $m \le |Q\times Q|$ & is a finite set of labels\\
% $T \subseteq Q\times \Sigma \times Q$ & is a transition relation
% \end{tabular}
% \end{center}
% with the following semantics: $q\stackrel{\sigma}{\rightarrow}q' \in T$ iff either
% \begin{enumerate}
% \item $\sigma \in \Sigma_a$ and

%  $\left(\exists p \right)\left(\forall z \in \{z\mid
%   V_M(z,p)=q\}\right)\left(\forall z'
%   \in \{z'\mid (z',p) \models \text{\textsc{Post}}(\sigma)\} \right)\left[(z,p) \models 
%   \text{\textsc{Pre}} (\sigma), V_M(z',p)=q'\right]$, or
% \item $\sigma \in \Sigma\setminus \Sigma_a$ and

%  $\left(\exists p
%   \right)\left(\forall z \in \{z\mid V_M(z,p) =q \}\right)(\exists
%   p'\in s(z,p), \sigma'\in \Sigma_a) \left[V_M(z,p') = q',\,
%     (z,p')\models \text{\textsc{Pre}}(\sigma') \right]$.
% \end{enumerate}
\end{definition}

It will be shown in Section~\ref{refine} that $H_a$ and $A(H_a)$ are
linked through an equivalence relation -- \emph{observable (weakly)
  simulation} relation.  Broadly speaking, the sequences (strings in
${\Sigma_a}^\ast$) of discrete states which $H_a$ visits starting from
$[z,p,\sigma]$ can be matched by a word $w$ such that
$T\big(V_M(z,p),w\big)$ is defined in $A(H_a)$, and vice versa, modulo
symbols in $\Sigma \setminus \Sigma_a$ that are thought of as
\emph{silent}.  When a \ac{sa} moves from state $q$ to state $q'$
through a series of consecutive transitions among which only one is
labeled with $\sigma \in \Sigma_a$ and all others in $\Sigma \setminus \Sigma_a$, then we say that the \ac{sa}
takes a \emph{composite} transition from $q$ to $q'$, labeled with
$\sigma$, and denoted $q \stackrel{\sigma}{\leadsto} q'$.

%Observable simulation relations establish inclusion of languages
%between two \ac{sa}s, modulo their silent transitions.

\begin{definition}[Weak (observable) simulation \cite{Faron96}]
  Consider two (labeled) semiautomata over the same input
  alphabet $\Sigma$, $A_1=\langle Q_1,\Sigma, \leadsto_1\rangle$
  and $A_2=\langle Q_2,\Sigma,\leadsto_2\rangle$, and let
  $\Sigma_\epsilon \subset \Sigma$ be a set of labels associated with
  silent transitions.  An
  ordered binary relation $\mathfrak R$ on $Q_1 \times Q_2$ is a
  \emph{weak (observable) simulation} if: \begin{inparaenum}[(i)] 
  \item $\mathfrak R$ is total,
  i.e., for any $q_1 \in Q_1$ there exists $q_2 \in Q_2$ such that
  $(q_1,q_2) \in \mathfrak{R}$, and \item for every ordered pair $ (q_1,q_2)
  \in {\mathfrak R}$ for which there exists $q_1'$ such that
$q_1\stackrel{\sigma}{\leadsto_1} q_1' $, then 
 $\exists \; (q_1',q_2') \in \mathfrak R : 
 q_2\stackrel{\sigma}{\leadsto_2} q_2' $.
\end{inparaenum}
Then $A_2$ weakly simulates $A_1$ and we write
$A_2 \gtrsim A_1$.
\end{definition}

Task specifications for hybrid systems (and transition systems, by
extension) may be translated to a Kripke structure \cite{clarke} (see
\cite{Belta2007} for examples), which is basically a \ac{sa} with
marked initial states, equipped with a labeling function that maps a
state into a set of logic propositions that are true at that state.
In this paper we also specify final states, and allow the labeling
function to follow naturally from the semantics of the valuation map.
We thus obtain a \ac{fsa} $\mathcal{A}_s = \langle Q_s, \Sigma_s,T_s,
I_s, F_s \rangle$, where $I_s$ and $F_s$ denote
the subsets of initial and final states, respectively.  Given the dynamic environment, a system \big($H_a$ or $A(H_a)$\big)
\emph{satisfies} the specification $\mathcal{A}_s$ if the interacting
behavior of the system and the environment forms a word that is
accepted in $\mathcal{A}_s$.%  In the following context, the abstraction
% of the system refers to as \ac{sa} $A_1$, the model of environment is
% $A_2$ and the task is $\mathcal{A}_s$, that can be related with the
% components in Fig.~\ref{fig:hybrid-learning}.

\subsection{Games on Semiautomata}

Here, we follow for the most part the notation and terminology of
\cite[Chapter 4]{Pin}.
Let $A_1=\langle Q_1,\Sigma_1,T_1\rangle$ represents the dynamics of
player 1, and $A_2 = \langle Q_2,\Sigma_2, T_2\rangle$ those of player
2.  We define the set $I_i \subseteq Q_i$ as the set of
\emph{legitimate initial states} of $A_i$, for $i=1,\,2$ respectively,
but we do not specify final states in these two \ac{sa}. The language
\emph{admissible} in $A_i$ is $\mathcal{L}(A_i)=\bigcup_{q_0\in I_i}
\bigcup_{q\in Q_i}\{w\mid T_i(q_0,w) =q\}$, which essentially includes
all possible sequences of actions that can be taken in $A_i$. Let
$\Lambda=\Sigma_1 \cup \Sigma_2$.  Define an (infinite) \emph{game}
\cite{Pin} $\mathcal{G}(\Phi)$ on $\Lambda$ as a set $\Phi \subset
\Lambda^\omega$ of infinite strings consisted of symbols from the two
alphabets $\Sigma_1$ and $\Sigma_2$ taken in turns. A \emph{play} is
an infinite string $w = \sigma_1 \sigma_2 \cdots \in \Lambda^\omega$.
Players take turns with player 1 playing $\sigma_1$ first by default.
In this paper we assume that players can give up their turn and
``play" a generic (silent) symbol $\epsilon$, i.e.\ $\epsilon \in
\Sigma_i$ and $T_i(q,\epsilon) = q$, $\forall \,q \in Q_i$.  A pair of
symbols $\sigma_{2i-1} \sigma_{2i}$ for $i = 1,\ldots$ denotes a
round, with any one of the two symbols being possibly equal to
$\epsilon$.  We say that player 1 wins the game if $w \in \Phi$; if
not, then player 2 wins.
A \emph{strategy} for player $i$ in game $\mathcal{G}(\Phi)$ is a
function $\mathsf{S}_i: \Lambda^\ast \to \Sigma_i$.  Player 1 (2)
\emph{follows} strategy $\mathsf{S}_1$ (respectively, $\mathsf{S}_2$)
in a play $w=\sigma_1 \sigma_2 \cdots$ if for all $n \ge 1$,
$\sigma_{2n-1} = \mathsf{S}_1(\sigma_1 \sigma_2 \cdots \sigma_{2n-2})$
\big(respectively, $\sigma_{2n} = \mathsf{S}_2(\sigma_1 \sigma_2
\cdots \sigma_{2n-1})$\big).  A strategy for player 1 is a
\emph{winning strategy} $\mathsf{WS}_1$ if all strings $w = \sigma_1
\sigma_2 \cdots $ that satisfy $ \sigma_{2n-1} =
\mathsf{WS}_1(\sigma_1 \sigma_2 \cdots \sigma_{2n-2}),\, \forall n \ge
1$, belong in $\Phi$. Winning strategies for player 2 are defined
similarly. If one of the players has a winning strategy, then the game
is \emph{determined}.  % All open games are determined.

%===============================================

\section{Game Theoretic Approach to Planning}
\label{section:analysis}

\subsection{Constructing the game} 

Consider a hybrid agent having to satisfy a task specification, encoded in
a \ac{fsa} $\mathcal{A}_s$.
Assume that this agent is operating in an unknown environment.
In the worst case, this environment is controlled by an intelligent adversary
who has full knowledge of the agent's capabilities.  The adversary is
trying to prevent the agent from achieving its objective.  The behavior of the
environment is still rule-based, i.e.\ subject to some given dynamics, although
this dynamics is initially unknown to the agent.

Assume that the agent has been abstracted to a \ac{sa} $A_1$ (player
1) and the dynamics of the environment is similarly expressed in
another \ac{sa} $A_2$ (player 2).  Without loss of generality, we
assume the alphabets of $A_1$ and $A_2$ are disjoint,
i.e.\ $\Sigma_1\ne \Sigma_2$. In this game, the agent is not allowed to
give up turns ($\epsilon \notin \Sigma_1$) but the adversary that
controls the environment can do so ($\epsilon \in \Sigma_2$).  For
two-player turn-based games, the actions of one player may influence
the options of the other by forbidding the latter to initiate certain
transitions.  To capture this interaction mechanism we define the
\emph{interaction functions} $U_i: Q_i\times Q_j \rightarrow
2^{\Sigma_j}, (i,j)\in\{(1,2),(2,1)\}$.  An interaction function $U_i$
maps a given pair of states $(q_i,q_j)$ of players $i$ and $j$, to the
set of actions player $j$ is not allowed to initiate at state $q_j$.

We now define a \ac{sa} that abstractly captures the dynamics of interaction between
the two players, by means
of a new operation on \ac{sa} which we call the \emph{turn-based product}. 
An intersection of the turn-based product with the task specification yields the
representation of the game and further allows us to compute the
strategy for the agent.

\begin{definition}[Turn-based product]
\label{def:turnbased-prod}
  Given two \ac{sa}s for players $A_1 = \langle Q_1,\Sigma_1,T_1\rangle
  $ and $A_2= \langle Q_2,\Sigma_2,T_2\rangle $ with the sets of
  legitimate initial states $I_1, \, I_2$ and interacting functions
  $U_1,\,U_2$, their turn-based product $P=\langle Q_p, \Sigma_1\cup \Sigma_2, T_p \rangle$ is
  a \ac{sa} denoted $A_1 \circ A_2$,  and is defined as follows:
\begin{itemize}
\item $Q_p = Q_1 \times Q_2 \times \left\{\bm 0, \bm 1\right\}$, where
  the last component is a Boolean variable $c \in \left\{\bm 0, \bm
    1\right\}$ denoting who's turn it is to play: $c=\bm 1$ for player
  1, $c=\bm 0$ for player 2.
\item $T_p\big((q_1,q_2,c),\sigma\big) =(q_1', q_2, \bm 0)$ if 
$c=\bm 1, \; q_1'= T_1(q_1,\sigma)$, with $\sigma \notin
     U_2(q_2,q_1)$ and  $T_p\big((q_1,q_2,c),\sigma\big)$ 
	$=(q_1,q_2',\bm 1)$ if  $c=\bm 0, \; q_2'= T_2(q_2,\sigma)$, with $\sigma \notin
    U_1(q_1,q_2)$.
\end{itemize}
\end{definition}

Assuming player 1 is the first one to make a move, the set of
legitimate initial states in $P$ is $I_1\times I_2\times \{\bm 1\}$
and the language \emph{admissible} in $P$ is $\mathcal{L}(P) =
\bigcup_{q_0 \in I_1\times I_2 \times \{\bm 1\}}\;\bigcup_{ q\in Q_p}
\left\{ w \mid T_p(q_0,w)=q\right\}, $ the set of all possible plays
between two players.  Note that if one includes the silent action
$\epsilon$ in $\Sigma_i$ for $i=1,2$, the players may not necessarily
play in turns---as in the specific case of agent-environment
interaction considered here.  The product operation is still
applicable as defined.

% The turn-based product $P$ gives snapshots of the game stages; just
% like the phases in the chess puzzles one might see in newspapers
The turn-based product $P$ gives snapshots of different stages in a
game.  It does not capture any of the game history that resulted in
this stage.  Often, task specifications encoded in $\mathcal{A}_s$
involve a history of actions, and thus the winning conditions for
player 1 cannot be encoded in $P$ by simply marking some states as
final.  We overcome the lack of memory in $P$ by taking its product
with $\mathcal{A}_s$.  Taking the product is suggested by the fact
that player 1 can win the game (i.e.\ agent can satisfy the
specification) only if $L(\mathcal{A}_s)\cap  \mathcal{L}(P) \ne \emptyset$.  The technical complication is that the two
terms in this product are heterogeneous: one is a \ac{sa} and the
other is a \ac{fsa}.  We resolve this by transforming the \ac{sa} into
a \ac{fsa} and applying the standard product operation; and the result
is what we call the \emph{game automaton}.

\begin{definition}[Game automaton]
  \label{def:gamefsa}The game automaton is a \ac{fsa} defined as
  $\mathcal{G}= \mathcal{P} \times \mathcal{A}_s = \langle Q, \Sigma,
  T, Q_0,F \rangle$, where $\mathcal{A}_s = \langle
  Q_s,\Sigma,T_s,I_s,F_s\rangle$ is a \ac{fsa} encoding the winning
  conditions for player 1, and $\mathcal{P}$ is a \ac{fsa} obtained
  from the turn-based product $P = A_1\circ A_2$ by defining the set
  of initial states of $\mathcal{P}$ as the legitimate initial
  states $I_1\times I_2 \times \{\bm 1\}$, % $\{ (q_1,q_2,\bm 1) \in Q_p :
  % q_1 \in I_1,\; q_2 \in I_2\}$
   and marking all other states as final.  The set of initial states
  for $\mathcal{G}$ is defined as $Q_0 =\{ (q_1,q_2,\bm 1,q_{0s})\mid
  q_1\in I_1, q_2 \in I_2,\; q_{0s} \in I_s \}$.  The set of final
  states for $\mathcal{G}$ is given by $F = \{ (q_1,q_2, \bm 0, q_s)
  \mid q_s \in F_s\}$.
\end{definition}

It follows (from the fact that the language of $\mathcal{G}$ is
regular) that the game defined by $\mathcal{G}$ is a reachability game
\cite{Thomas2002}, and therefore
it is determined.  Note that the final states of $\mathcal{G}$ are
exactly those in which player 1 wins the game.  On \ac{fsa}
$\mathcal{G}$, we define the \emph{attractor} of $F$, denoted
$\mathsf{Attr}(F)$, which is the largest set of states $W\supseteq F$
in $\mathcal{G}$ from where player 1 can force the play into $F$. It
is defined recursively as follows.  Let $W_0 = F$ and set
\begin{multline}
  \label{eq:attractor}          
  W_{i+1} := W_i \cup \{ q \in Q \mid q = (q_1,q_2,\bm 1,q_s), \text{
    and }
  \exists \sigma \in \Gamma(q) :  T(q,\sigma) \in W_i\} \\
  \cup \{q \in Q \mid q=(q_1,q_2,\bm 0,q_s), \text{ and } \forall
  \sigma \in \Gamma(q) :T(q,\sigma) \in W_i \}\enspace.
\end{multline}
The function $\rho: Q \to \mathbb{N}; \; \rho(q) \mapsto 
\min\{i \ge 0 \mid q \in W_i\}$ is called the \emph{rank function} of the game.

Since $\mathcal{G}$ is finite, there exists the smallest $m\in \nat $
such that $W_{m+1}=W_m$.  Then $\mathsf{Attr}(F) = W_m$. Moreover,
because $\mathcal{G}$ is determined, the complement of
$\mathsf{Attr}(F)$ in $Q$ forms a \emph{trap} for player 1; it
contains all the states at which player 2 can prevent player 1 from
winning the game. $\mathsf{Attr}(F)$ can be computed in time
$\mathcal{O}(n_1+n_2)$ where $n_1=|Q|$ and $n_2$ is the number of
transitions in $\mathcal{G}$.

\subsection{Computing a winning strategy}
\label{section:winning}

The following statement is straightforward.
\begin{theorem}
\label{thm:win} Player 1 has a
  winning strategy iff  $ \mathsf{Attr}(F) \cap Q_0 \ne \emptyset$. 
\end{theorem}
\begin{proof}
  If $\mathsf{Attr}(F) \cap Q_0 \ne\emptyset$, the winning strategy of
  player 1 can be defined as a map $\mathsf{WS}_1:Q \to
  2^{\Sigma_1}$, so that for $q = (q_1,q_2,\bm 1,q_s)$, the image of
  this map is
  $\mathsf{WS}_1(q)=\{\sigma \mid T(q,\sigma)\in
  \mathsf{Attr}(F)\}$. If the game starts at $q_0\in
  \mathsf{Attr}(F) \cap Q_0 $, by exercising $\mathsf{WS}_1$, player 1
  ensures that subsequent states are within its attractor.
\end{proof}

We refer to $\mathsf{Attr}(F) \cap Q_0$ as the set of \emph{winning
  initial states} of $\mathcal{G}$.  Notice that strategy
$\mathsf{WS}_1$ keeps player 1 in its attractor, ensuring that it can
win the game, but does not necessarily guide it into winning.  To
compute an \emph{optimal} winning strategy---one that wins the game
for player 1 in the least number of turns---we partition $W_m$ into a
set of subsets $V_i$, $i=0,\ldots,m$ in the following way: let $V_0 =
W_0= F$ and set $V_i := W_i \setminus W_{i-1}$, for all $i \in
\{1,\ldots,m\}$.  The sets $V_i$s partition the attractor into layers,
according to the rank of the states that are included.  That is,
$\forall q \in V_i$, $\rho(q) = i$ and thus the $\{V_i\}_{i=1}^m$ partition is
the one induced by the ranking function.  We can then prove the
following sequence of statements.

Once the game is in $\mathsf{Attr}(F)$, all the actions of player 2, and
some of player 1 strictly decrease the rank function:
\begin{lemma}
  \label{lm:Vprop}
  For each $q\in V_{i+1}$, $i=0,\ldots,m-1$, if $c = \bm 1$, then
  $\exists\; \sigma \in \Sigma_1 \cap \Gamma(q)$ such that
  $T(q,\sigma) \in \mathsf{Attr}(F)$, it is $\rho\big(T(q,\sigma)\big)
  = i$.  If $c = \bm 0$, then $\forall \sigma \in \Sigma_2 \cap
  \Gamma(q)$, such that $\rho\big(T(q,\sigma)\big) = i$.
\end{lemma}
\begin{proof} Let $q\in V_{i+1}$.  According to \eqref{eq:attractor},
  either \begin{inparaenum}[(a)] \item $c = \bm 1$ and so
    $T(q,\sigma)\in W_i$ for some $\sigma \in \Gamma(q)$, or
  \item $c = \bm 0$ and $T(q,\sigma)\in W_i$, $\forall\, \sigma \in
    \Gamma(q)$ \end{inparaenum}.  We show the argument for case (a)
  when $c = \bm 1$ by contradiction: suppose there exists $k<i$, so
  that $T(q,\sigma)\in V_k$---by construction \eqref{eq:attractor} we
  already have $k\le i$. Then according to \eqref{eq:attractor}, $q$
  belongs to $V_{k+1} $.  But since the sets $V_i$ partition
  $\mathsf{Attr}(F)$, $V_{k+1}$ and $V_{i+1}$ are disjoint.  Therefore
  $q$ cannot be in $V_{i+1}$ as assumed in the statement of the Lemma.
  Thus, when $c = \bm 1$, all actions that enable the player to remain
  in its attractor in fact move it only one (rank function value) step
  closer to the winning set.  A similar contradiction argument applies
  to case (b) when $c = \bm 0$: Assume that all $\sigma \in \Sigma_2
  \cap \Gamma(q)$ yield $T(q,\sigma) \in V_j$ for some $j < i$.  Let
  $k = \max_{q' \in T(q,\sigma)} \; \rho(q')$.  Then $ i> k \ge j$,
  which means that $k+1 < i+1$.  In the same way we arrive at $q
  \notin V_{i+1}$ which is a contradiction.
\end{proof}

Informally, actions of player 1 from $V_{i+1}$ cannot take the game
any closer to $F$ than $V_i$% , because then the particular state
% would have had a lower rank than $i+1$
.  This implies that the
rank of a state expresses the total number of turns in which player 1 can win the game from
that state.
\begin{proposition}
\label{prop:pathproperty}
For each $q\in V_i$, there exists at least one word $w\in
L(\mathcal{G})$, with $|w| = i$ such that $T(q,w)\in F$.
\end{proposition}
\begin{proof}
  We use induction, and we first prove the statement for $i=1$. For
  each $q =\left(q_1,q_2,\bm 1,q_s\right)\in V_1$,
  Lemma~\ref{lm:Vprop} suggests that at least one action of player 1
  which keeps it in the attractor, actually sends it to $V_0 = F$.  So
  for $i=1$ the plays in which player 1 wins have length one.  Now
  suppose the statement holds for $i = n$; we will show that also
  holds for $i=n+1$.  According to Lemma \ref{lm:Vprop}, for each $q
  \in V_{n+1}$, $\forall\; \sigma \in \Sigma_2 \cap \Gamma(q)$ (player 2 taking its
  best action) or for at least one $\sigma \in \Sigma_1 \cap
  \Gamma(q)$ (player 1 taking its best action) we will have
  $T(q,\sigma)\in V_n $.  In other words, if both players play their
  best, the rank of the subsequent state in the game automaton will be
  $n$.  Inductively, we conclude the existence of a path of length $n$
  in $\mathcal{G}$ starting at $q \in V_{n}$ and ending in $q' \in V_0
  = F$.
\end{proof}

\begin{proposition}
\label{prop:optimal-moves}
  Suppose $q_0= (q_1,q_2,\bm 1,q_{s0})$ and
  that $\rho
(q_0) = k \le m$.   Then player 1
  can win the game in at most $k$ rounds following the strategy
  $\mathsf{WS}^*_1$, defined as
 \begin{equation}
\label{eq:shortestws}
\mathsf{WS}_1^*(q)=\left\{\sigma
  \mid T(q,\sigma) \in
  V_{i-1}, \;   q\in V_{i},
  \; i\ge 1\right\}
   \enspace.
\end{equation}
\end{proposition}
\begin{proof}
  Given a state $q = (q_1,q_2,\bm 1,q_s)\in V_i$, $\mathsf{WS}^*_1$
  allows player 1 to force the game automaton to reach a state in $
  V_{i-1}$ by picking action $\sigma^\ast$ such that
  $T(q,\sigma^\ast)=q'$ where $q' \in V_{i-1}$
  (Lemma~\ref{lm:Vprop}). At $q'$, $c = \bm 0$. Any action of player 2
  takes the game automaton to a state $q''\in V_j$ for $j\le i-2$.  In
  fact, the best player 2 can do is to delay its defeat by selecting
  an action $\sigma$ such that $j=i-2$ (Lemma~\ref{lm:Vprop}). An
  inductive argument can now be used to complete the proof.
\end{proof}

%\input{buchigame}

%===============================================

\section{Learning through Grammatical Inference}
\label{section:GI}
In Section~\ref{section:analysis} it was shown that the agent can
accomplish its task iff \begin{inparaenum}[(a)]
\item it has full knowledge of
  the environment, and 
\item the game starts at the winning initial state in
  $\mathsf{Attr}(F)\cap Q_0$\end{inparaenum}. The problem to be
answered in this section is if the environment is (partially) unknown
but rule-governed, how the agent plans its actions to accomplish its
task. By assuming the language of the environment is \emph{learnable}
by some \ac{gim}, we employ a module of grammatical inference to solve
this problem.
\subsection{Overview}
The \emph{theory of mind} of an agent refers to the ability of the
agent to infer the behavior of its adversary and further its own
perception of model of the game
\cite{FrithFrith-2003,PremackWoodruff-1978}.  In the context of this
paper, the agent initially has no prior knowledge of the capabilities
of its adversary and plans a strategy based on its own hypothesis for
the adversary.  Therefore, although the agent makes moves which keep
it inside the \emph{hypothesized} attractor, in reality these moves might 
take it outside the \emph{true} attractor. Once the agent has departed its true
attractor, then it is bound to fail since the adversary knows the true
nature of the game and can always prevent the agent from fulfilling
its task.

An agent equipped with a \ac{gim} is able to construct an increasingly
more accurate model of the behavior of its adversary through
consequent games (Fig.~\ref{fig:learning}).  The expected result is
that as the agent refines the model it has for its environment and
updates its ``theory of mind,'' its planning efficacy increases.  We
expect that after a sufficient number of games, the agent should be
able to devise strategies that enable it to fulfill its task
irrespective of how the adversary proceeds.  This section presents the
algorithms for constructing and updating this model.

\subsection{Assumptions and Scope}

In the agent-environment game, the behavior of the unknown environment
becomes a positive presentation for the learner.  The hypothesis
obtained by the learner is used for the agent to recompute
the game automaton and the attractor as described in
Section~\ref{section:analysis}. It is therefore guaranteed that the agent's
hypothesis of the unknown environment will eventually converge to the
true abstract model of the environment,
provided that \begin{inparaenum}[(i)]
\item the true model lies within the class of models inferable by the
  learner from a positive presentation, and
\item the unknown environment's behavior suffices for a correct
  inference to be made (for example if a characteristic sample for
  the target language is observed). \end{inparaenum} 

We make the following assumption on the structure of the unknown
discrete dynamics of the adversarial environment:

\begin{assumption}
  The language admissible in the \ac{sa} $A_2$ of the adversarial
  environment (player 2) is identifiable in the limit from
  positive presentation.
\end{assumption}

Although the results we present extend to general classes of systems
generating string extension languages, for clarity of presentation we
will focus the remaining discussion on a particular subclass of string
extension languages, namely \emph{Strictly $k$-Local} languages
(\ac{sl}$_k$) \cite{Luca1980}, which has been defined in
Section~\ref{section:grinf-pre}.

\subsection{Identifying the Class of the Adversary's Behavior}
\label{subsec:characterizetheenv}

As suggested by Theorem \ref{thm:sl}, in order to identify the
behavior of the adversary, which is expressed in form of a language,
the agent must know whether this language is \ac{sl} and if it is, for
which $k$ in \ac{sl} hierarchy. We assume the information is provided
to the agent before the game starts. We employ the algorithm in
\cite{Caron1998} adapted for \ac{sa} to check whether a given
\ac{sa} admits a \ac{sl}
language.\footnote{This algorithm works with the graph representation
  of a \ac{fsa} and therefore it is not necessary to designate the
  initial states.} In what follows we provide a method for
determining the natural number $k$:

For some $k >0$, consider a (non)-canonical \ac{fsa} that accepts
$\Sigma^\ast$: $\mathcal{D}_k =\langle Q_D,\Sigma,T_D,\{\lambda\},
F_D \rangle$, where
\begin{inparaenum}[(i)]
\item $Q_D = \mathsf{Pr}^{\le k-1}(\Sigma^\ast)$;
\item $T_D(u,a)= \mathsf{Sf}^{= k-1}(ua)$ iff $|ua|\ge k-1$
and $ua$ otherwise;
\item $\lambda$ is the initial state, and \item $F_D=Q_D$ is the set
  of final states (all states are final).
\end{inparaenum}
We refer to $\mathcal{D}_k$ as the $\text{\ac{sl}}_k$-\ac{fsa}
  for $\Sigma^\ast$.  It is shown \cite{Heinz-thesis} that for a
given a \ac{sl}$_k$ language with grammar $\mathfrak{G}$, a
(non)-canonical \ac{fsa} accepting $L(\mathfrak{G})$ can be obtained
by removing some transitions and the finality of some of the
states\footnote{Removing finality of a state $q$ in \ac{fsa}
  $\mathcal{A}$ means to remove $q$ from the set of final states in
  $\mathcal{A}$.}  in $\mathcal{D}_k$. We call the \ac{fsa} of a
\ac{sl}$_k$ language $L(\mathfrak{G})$ obtained in this way, the
\ac{sl}$_k$-\ac{fsa} of $L(\mathfrak{G})$.  Figure~\ref{fig:sl3}
shows a \ac{sl}$_3$-\ac{fsa} for $\Sigma^\ast$, with $\Sigma=
\{a,b\}$.  Figure~\ref{fig:sl3fsa} shows another \ac{sl}$_3$ grammar
that generates the language given by the string extension grammar
$\mathfrak{G} = \{\rtimes aa, \rtimes ab, aab, aaa, aba,
ba\ltimes\}$. For example, $aaba \in L(\mathfrak{G})$ because
$\mathfrak{f}_{3}(\rtimes aaba \ltimes) = \{\rtimes aa, aab,
aba,ba\ltimes\} \subset \mathfrak{G}$. Yet $aababa \notin
L(\mathfrak{G})$ as $\mathfrak{f}_{3}(\rtimes aababa \ltimes) =
\{\rtimes aa, aab, aba,bab, ba\ltimes\} \nsubseteq \mathfrak{G}$, in
fact the $3$-factor $bab\notin \mathfrak{G}$.
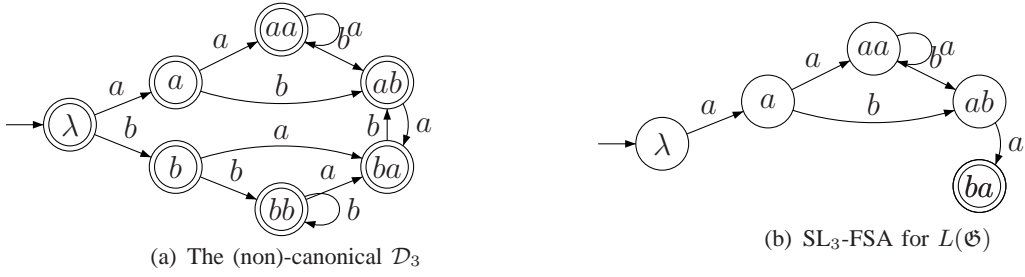
\begin{figure}[H]
 \centering
  \subfigure[The (non)-canonical $\mathcal{D}_3$]{\label{fig:sl3}
\begin{minipage}{.45\textwidth}
    \unitlength=4pt
    \begin{picture}(30, 20)(-6,4)
    \gasset{Nw=5,Nh=5,Nmr=2.5,curvedepth=0,loopdiam=3}
    \thinlines
    \node[Nmarks=ir,iangle=180,fangle=180](A0)(0,12){$\lambda$}
    \node[Nmarks=r](A1)(10,8){$b$}
    \node[Nmarks=r](A2)(10,16){$a$}
    \node[Nmarks=r](A3)(20,21){$aa$}
    \node[Nmarks=r](A4)(20,4){$bb$}
    \node[Nmarks=r](A5)(30,16){$ab$}
    \node[Nmarks=r](A6)(30,8){$ba$}
   \drawedge(A0,A1){$b$}
   \drawedge(A0,A2){$a$}
   \drawedge(A2,A3){$a$}
   \drawedge[curvedepth=-2](A2,A5){$b$}
   \drawedge(A1,A4){$b$}
   \drawedge[curvedepth=2](A1,A6){$a$}
   \drawloop[loopangle=0](A3){$a$}
   \drawloop[loopangle=0](A4){$b$}
   \drawedge[curvedepth=2](A5,A6){$a$}
   \drawedge(A6,A5){$b$}
   \drawedge(A4,A6){$a$}
   \drawedge(A3,A5){$b$}
  \end{picture}
  \vspace{0.2in}
\end{minipage}
\label{figure:D3}
}
\subfigure[\ac{sl}$_3$-\ac{fsa} for $L(\mathfrak{G})$]{\label{fig:sl3fsa}
\begin{minipage}{.45\textwidth}
  \unitlength=4pt
    \begin{picture}(30, 20)(-6,4)
    \gasset{Nw=5,Nh=5,Nmr=2.5,curvedepth=0,loopdiam=3}
    \thinlines
    \node[Nmarks=i,iangle=180,fangle=180](A0)(0,12){$\lambda$}
    %\node(A1)(10,8){$b$}
    \node(A2)(10,16){$a$}
    \node(A3)(20,21){$aa$}
    %\node[Nmarks=r](A4)(20,4){$bb$}
    \node(A5)(30,16){$ab$}
    \node[Nmarks=r](A6)(30,8){$ba$}
    \node(A6)(30,8){$ba$}
  % \drawedge(A0,A1){$b$}
   \drawedge(A0,A2){$a$}
   \drawedge(A2,A3){$a$}
   \drawedge[curvedepth=-2](A2,A5){$b$}
  % \drawedge(A1,A4){$b$}
   %\drawedge[curvedepth=2](A1,A6){$a$}
   \drawloop[loopangle=0](A3){$a$}
   %\drawloop[loopangle=0](A4){$b$}
   \drawedge[curvedepth=2](A5,A6){$a$}
   %\drawedge(A6,A5){$b$}
   %\drawedge(A4,A6){$b$}
   \drawedge(A3,A5){$b$}
  \end{picture}
\end{minipage}
}
\caption{The (non)-canonical \ac{fsa} $\mathcal{D}_3$ accepting
  $\Sigma^\ast$ for $\Sigma = \{a,b\}$ (left) and the
  \ac{sl}$_3$-\ac{fsa} obtained for $L(\mathfrak{G})$, where
  $\mathfrak{G} = \{\rtimes aa, \rtimes ab, aab,aaa, aba,
  ba\ltimes\}$, after removing transitions and the finality of some
  states (right).}
%\label{fig:slkfsa}
\vspace*{-5ex}
\end{figure}

In a \ac{fsa}, we say $q\in Q$ is at level $i$ iff $i=\min \{|w|\mid
w\in \Sigma^\ast, T(q_0,w)=q\}$, where $q_0$ is an initial state. The
function $\gamma : Q\rightarrow \mathbb{N}$ maps a state $q$ to its
level.  Now we can state the following.
\begin{lemma}
  If a canonical \ac{fsa} $\mathcal{C}= \langle Q_c,\Sigma, T_c,q_{0c},
  F_c\rangle$ accepts a \ac{sl} language $L$ for some $k$ where $k$ is the
  smallest number such that $L(\mathcal{C})\in \text{\ac{sl}}_k$, then $k \le
  \max_{q\in F_c} \gamma(q) +1$.
\end{lemma}
\begin{proof}
  Let $\mathfrak{G}$ be a \ac{sl}$_k$ grammar that generates $L$.  
  Then we can generate a
  (non)-canonical \ac{fsa} $\mathcal{B} = \langle Q_b,\Sigma, T_b, \{\lambda\},
  F_b\rangle$ by removing transitions and finality of nodes from $\mathcal{D}_k$.
  Let $q^\ast = \arg \max_{q\in F_c} \gamma(q)$ be a state in $\mathcal{C}$
  furthest from the initial state, let $n = \gamma(q^\ast)$ be its level,
  and $w = w_1w_2\cdots w_n$ be a word that brings $\mathcal{C}$ to
  state $q^\ast = T(q_{0c},w)$.  \ac{fsa}s $\mathcal{B}$ and $\mathcal{C}$
  accept the same languages, so $w\in
  L(\mathcal{C})$ iff $w\in L(\mathcal{B})$.  In $\mathcal{B}$, however,
  we can compute a $k$, because   
  $T_b(\lambda,w) = \mathsf{Sf}^{=k-1}(w)\in F_b$
  with $k-1 \le n$, i.e.\ $k \le n+1$.
\end{proof}

Though we can only obtain an upper bound $k_{\max}=\max_{q\in F_c}
\gamma(q) +1$ on the smallest $k$ (in the worst case this bound is
$|Q_c|$), the hierarchy of \ac{sl} language class given by
Theorem~\ref{thm:hierarchy} guarantees that this upper bound
$k_{\max}$ is sufficient for us to obtain a correct
\ac{sl}$_{k_{\max}}$ grammar that generates the exact language
presented to the learner, irrespectively if this language can also be 
generated by a \ac{sl}$_k$ grammar for some $k \le k_{\max}$. For example, for the
language accepted by the \ac{fsa} in Fig.~\ref{fig:sl3fsa}, we can
also obtain a \ac{sl}$_4$ grammar $\mathfrak{G}'= \left\{\rtimes a a
  a, \rtimes a ba, \rtimes aab, aaba, aaab, aba\ltimes\right\}$ and it
can be verified that $L(\mathfrak{G}') = L(\mathfrak{G})$.

\subsection{Learning the Adversary's Dynamics}
\label{section:learning-the-sw}
Before the game starts, player 1 is informed that the behavior of its
adversary is a \ac{sl}$_k$ language for some known $k$ and the
adversary can always give up a turn, i.e. $\epsilon \in \Sigma_2$.
With this knowledge, player 1 builds a \ac{sl}$_k$-\ac{fsa} for
$\left\{\Sigma_2\setminus \{ \epsilon\} \right\}^\ast$. Then, by
unmarking initial and final states and adding a self-loop labeled
$\epsilon$ at each state, it obtains an initial model of its adversary
$A_2^{(0)} = \langle Q_2,\Sigma_2 ,T_2\rangle$.

In the course of game, player 1 (agent) records the continuous
sequence of actions of player 2 (the environment). This amounts to a
presentation $\phi$ of the form: $ \phi(0)= \lambda, \,
\phi(i+1)=\phi(i)\sigma,\, i\ge 1, i\in \mathbb{N},$ for some $\sigma
\in \Gamma\big(T(q_{0},w) \big)\cap \Sigma_2 \ne \emptyset$ where $q_0
\in Q_0$ and $w\!\downharpoonright_{\Sigma_2}\; =
\phi(i)$.\footnote{This is a map $\downharpoonright_{\Sigma_2} :
  \Sigma^\ast \rightarrow \Sigma_2^\ast $. The image
  $w\downharpoonright_{\Sigma_2}$ is the string after removing all
  symbols in $w$ which are not in $\Sigma_2$.} The learning algorithm
is applied by player 1 to generate and refine the hypothesized model
of its adversary from the presentation $\phi$.

Since a \ac{fsa} for any \ac{sl}$_k$ grammar can be generated by
removing edges and finality of nodes in the \ac{sl}$_k$-\ac{fsa} for
$\Sigma^\ast$, then the \ac{sa} for player 2 can be obtained by just
removing edges in $A_2^{(0)}$. Due to this special property, we can
use an instrument with which the agent encodes new knowledge into the
hypothesized model for the adversary, namely, a \emph{switching
  function} $\mathrm{sw}$, which operates on a \ac{sa} (or \ac{fsa})
and either blocks or allows certain transitions to take place:
$\mathrm{sw}: Q_2\times \Sigma_2 \to \{0,1\}$, so that for $q \in
Q_2$, $\sigma \in \Gamma(q)$ only if $\mathrm{sw}(q,\sigma) = 1$.
Consequently, at round $i+1$, the incorporation of new knowledge for
$A_2$ obtained at round $i$ redefines
$\mathrm{sw}$. We assume a naive agent that starts its interaction
with the environment believing that the latter is static (has no
dynamics). That hypothesis corresponds to having
$\mathrm{sw}^{(0)}(q,\sigma) = 0$, $\forall \sigma \in
\Sigma_2\setminus \{\epsilon\}$ and
$\mathrm{sw}^{(0)}(q,\epsilon)=1,\forall q\in Q_2$.

% The positive presenation $\phi$ of $A_2$ corresponds to a sequence of
% sequences of actions taken by player $2$ in repeated games,
% i.e. $\phi(i)$ is the sequence of actions taken by player 2 at the
% $i$-th game. 

Note that $\phi(i)$ denotes the presentation up to round $i$. The
initialization of the game can be considered as a single round played
blindly by both players (without any strategy). Hence, if the game
starts with $\big((q_1,q_2,\bm 1) ,q_{0s}\big)$, it is equivalent to
have $\phi(1)= \sigma$, for which $T_2(\lambda,\sigma)=q_2$. Let
$\mathrm{sw}^{(i)}$ denote the refinement of $\mathrm{sw}$ made at
round $i$, suppose that at round $i+1$, the adversary plays
$\sigma'$. This suggests $\phi(i+1) = \phi(i) \sigma'$. Suppose $q_2 =
T_2(\lambda,\phi(i))$, then for all $q\in Q_2$ and
$\sigma\in\Sigma_2$, $\mathrm{sw}^{(i+1)}$ is defined by
\begin{equation}
\label{eq:updatesw}
\mathrm{sw}^{(i+1)}(q,\sigma) = \begin{cases} \mathrm{sw}^{(i)}(q,\sigma)
&\text{if } (q,\sigma) \neq (q_2,\sigma') \\
1 & \text{if } (q,\sigma) = (q_2,\sigma')
\end{cases}
\end{equation}
meaning that the transition from $q_2$ on input $\sigma'$ in $A_2$ is
now enabled. With a small abuse of notation, we denote the pair
$\left(A_2^{(0)}, \mathrm{sw}^{(i)}\right) = A_2^{(i)}$, read as the
\ac{sa} $A_2^{(0)}$ with switching function
$\mathrm{sw}^{(i)}$. Pictorially, $A_2^{(i)}$ is the \ac{sa} obtained
from $A_2^{(0)}$ by trimming the set of transitions which are switched
off ($\mathrm{sw}(\cdot)=0$).

Correspondingly, the game automaton in the initial theory of mind of
the agent is constructed as $\mathcal{G}^{(0)} =\langle
\mathcal{P}^{(0)} \times \mathcal{A}_s\rangle$ where
$\mathcal{P}^{(0)}$ is the \ac{fsa} obtained by $P^{(0)}= A_1 \circ
A_2^{(0)}$ after setting $ I_1 \times I_2 \times \{\bm 1\}$ as the set
of legitimate initial states, where $I_2=\{q\mid
T_2(\lambda,\sigma)=q,\sigma\in \Sigma_2\setminus \{\epsilon\}\}$, and
all other states in $P^{(0)}$ as final. By the construction of game,
the switching function associated with $A_2^{(i)}$ can be extended
naturally to $\mathcal{G}^{(i)} =
\left(\mathcal{G}^{(0)},\mathrm{sw}^{(i)} \right)$ by:
\begin{equation}
\label{eq:extendsw}
\enspace \forall q=(q_1,q_2,\bm{0},q_s), \sigma \in \Sigma_2,
\mbox{sw}^{(i)} (q,\sigma) = 1 \text{ (or } 0) 
\text{ in $\mathcal{G}^{(i)}$ iff } \mbox{sw}^{(i)}
(q_2,\sigma) = 1 \text{ (or } 0) \text{ in } A_2^{(i)} .
\end{equation}

With the extension of switching function, one is able to update the
game automaton without computing \emph{any} product during
runtime. This is because the structure of the game has essentially
been pre-compiled. This results in significant computational savings
during runtime, depending on the size of $A_2^{(0)}$.

This switching mechanism along with the extension from $A_2^{(i)}$ to
$\mathcal{G}^{(i)}$ can be applied to other classes of string
extension languages, in particular any class of languages describable
with \ac{fsa}s obtainable by removing edges and finality of states
from some deterministic \ac{fsa} accepting $\Sigma^\ast$.

\subsection{Symbolic Planning and Control}

With the theory of mind as developed in round $i$, and with the game
automaton at state $q$, the agent computes an optimal winning strategy
$\mathsf{WS}_1^\ast$ based on \eqref{eq:shortestws}, by setting $W_0 =
V_0 = F$ and iteratively evaluating \eqref{eq:attractor}, where
$\mathrm{sw}^{(i)}$ defined in $\mathcal{G}^{(i)}$ has to be taken
account of: for all $(q,\sigma) \in Q\times \Sigma$, if
$\mathrm{sw}^{(i)}(q,\sigma)=0$, then $\sigma \notin \Gamma(q)$.  The
computation terminates when the following condition is satisfied:
\begin{align} \label{eq:termination}
\exists \, m \in \mathbb{N} :&  & q \in W_m & & \lor & &
q \notin W_m = W_{m+1} \enspace.
\end{align}

When $q \in W_m$, $\mathsf{WS}_1^\ast$ can be computed at $q$.  Then
based on Proposition~\ref{prop:optimal-moves}, the strategy ensures
victory in at most $m$ turns.  The agent implements this strategy as
long as its theory of mind for the adversary remains valid, in other
words, no new transition has been switched on.  In the absence of new
information, the plan computed is optimal and there is no need for
adjustment.  If in the course of the game an action of the adversary,
which the current model cannot predict, is observed, then that model
is refined as described in Section~\ref{section:learning-the-sw}.
Once the new game automaton is available,
\eqref{eq:attractor}-\eqref{eq:shortestws} are recomputed, and
\eqref{eq:termination} is satisfied.

If instead $q \notin W_m = W_{m+1}$, then the agent thinks that $q \in
\mathsf{Attr}(F)^c$: the agent is in the trap of its adversary.  If
the adversary plays its best, the game is lost.  It should be noted
that this attractor is computed on the hypothesized game and may not be
the true attractor. Assuming that the adversary will indeed play
optimally, the agent loses its confidence in winning and resigns.
%-------------------------------------------------------------
In our implementation, when the agent resigns the game is restarted at
a random initial state $q_0\in Q_0$, but with the agent
\emph{retaining} the knowledge it has previously obtained about its
adversary. The guaranteed asymptotic convergence of a string extension
learner ensures that in each subsequent game, the agent increases its
chances of winning when initialized at configurations from which
winning strategies exist.  The adversary can always choose to prevent
the agent from learning by not providing new information, but by doing
so it compromises its own strategy.

The following section illustrates how the methodology outlined can be
implemented on a simple case study, and demonstrates the effectiveness
of the combination of planning with string extension learning.  As it
turns out, the identification of the adversary's dynamics is quite
efficient in relation to the size of $A_2$.

%======================================================

\section{Refinement on Hybrid Dynamics}
\label{refine}

Section~\ref{section:GI} established a methodology based on which the agent 
can concurrently learn and (re)plan an optimal strategy for achieving its objective,
in a partially known and adversarial environment.  This section addresses the problem
of implementing the optimal strategy on the concrete dynamics of the hybrid
agent $H_a$ as given in Definition~\ref{hybrid-agent}.

\begin{proposition}
\label{prop:pair}
Every transition labeled with $\tau \in \Sigma \setminus \Sigma_a$
must be followed by a transition labeled with some $\sigma \in \Sigma_a$,
i.e., every silent transition in $A(H_a)$ must be followed by an observable one.
\end{proposition}
\begin{proof}
  Assume, without loss of generality that the $\tau$ transition
  appears somewhere between two observable transitions $\sigma_1$,
  $\sigma_2$ $\in \Sigma_a$.  We will show that $\tau$ is the only
  silent transition that can ``fit'' between $\sigma_1$ and
  $\sigma_2$, in other words we can only have $q
  \stackrel{\sigma_1}{\to} q_1 \stackrel{\tau}{\to} q_2
  \stackrel{\sigma_2}{\to} q'$ for some $q$, $q_1$, $q_2$, and $q'$
  $\in Q$.  For that, note that by definition, $q$ must be such that
  for all $(z,p)$ giving $V_M(z,p) = q$, $(z,p) \models
  \text{\textsc{Pre}}(\sigma_1)$; similarly $q_1$ must be such that
  for all $(z',p)$ giving $V_M(z',p) = q_1$ we should have $(z',p)
  \models \text{\textsc{Post}}(\sigma_1)$.  Now suppose that there is
  another silent transition $\tau'$, in addition to $\tau$ between
  $\sigma_1$ and $\sigma_2$ and for the sake of argument assume that
  it comes right after $\tau$: $q \stackrel{\sigma_1}{\to} q_1
  \stackrel{\tau}{\to} q'' \stackrel{\tau'}{\to} q''' \cdots q_2
  \stackrel{\sigma_2}{\to} q'$.  With the $\tau$ transition following
  $\sigma_1$ we have by definition that there exists a $p'$ such that
  once the $\tau$ transition is completed it is $ (z',p') \models
  \text{\textsc{Pre}}(\sigma')$ for some $\sigma'$ $\in \Sigma_a$.
  Since $(z',p) \models \text{\textsc{Post}}(\sigma_1)$ and $ (z',p')
  \models \text{\textsc{Pre}}(\sigma')$, we have by
  Definition~\ref{hybrid-agent} that $H_a$ makes a transition from
  $(z',p',\sigma_1)$ to $(z',p',\sigma')$, and then the continuous
  component dynamics $f_{\sigma'}$ is activated yielding $z'
  \stackrel{\sigma'[p']}{\hookrightarrow} z''$ for some $(z'',p')
  \models \text{\textsc{Post}}(\sigma')$.  This time, with $ (z',p')
  \models \text{\textsc{Pre}}(\sigma')$ and $(z'',p') \models
  \text{\textsc{Post}}(\sigma')$, it follows that there is a $\sigma'$
  transition in $A(H_a)$ taking $q'' \stackrel{\sigma'}{\to} q'$, and
  $\sigma' = \sigma_2$ because there cannot be more than two
  observable transitions between $q$ and $q'$ by assumption.
  Therefore, $\tau$ is the only silent transition that must have
  occurred while $A(H_a)$ moved from $q$ to $q'$.
  \end{proof}

Due to Proposition~\ref{prop:pair}, without loss of generality we will
assume that a composite transition consists of a silent transition followed
by an observable transition, $q \stackrel{\sigma}{\leadsto} q'$ 
$\Longleftrightarrow$ $q \stackrel{\tau}{\to} q'' \stackrel{\sigma}{\to} q'$. 

\begin{theorem}
\label{thm:simulation}
Let $\Sigma_\epsilon = \Sigma \setminus \Sigma_a$, the hybrid agent
$H_a$ weakly simulates its induced semiautomaton $A(H_a)$ ($H_a
\gtrsim A(H_a)$) in the sense that there exists an ordered total
binary relation $\mathfrak{R}$ such that whenever $(q,z) \in
\mathfrak{R}$ and $q \stackrel{\sigma}{\leadsto}q'$ for some $q' \in
Q$, then $\exists z'\in \mathcal{Z}:
z\stackrel{\sigma[p]}{\hookrightarrow}z'$ such that $ (q',z')\in
\mathfrak{R}$.
\end{theorem}
\begin{proof}
  If $(q,z) \in \mathfrak{R}$, then there exists $p^0 \in \mathcal{P}$
  such that $V_M(z,p^0) = q$.  In general, $p^0 \neq p$.  Using the
  convention adopted above for the composite transition, we write
  $q
    \stackrel{\sigma}{\leadsto} q'$ $\Longleftrightarrow$ $q
  \stackrel{\tau}{\to} q'' \stackrel{\sigma}{\to} q'$ with $\sigma \in
  \Sigma_a$ and $\tau \in \Sigma\setminus \Sigma_a$.  The transition
  $q \stackrel{\tau}{\to} q''$, by definition, implies that for all
  $z$ such that $V_M(z,p^0) = q$, there exists $p\in s(z,p^0)$ and
  $\sigma' \in \Sigma_a$ such that $V_M(z,p) = q''$ with $(z,p)
  \models \text{\textsc{Pre}}(\sigma')$.  With $q''
  \stackrel{\sigma}{\to} q'$ assumed, we have by definition that for
  all $z$ such that $V_M(z,p) = q''$ it should be $V_M(z',p) = q'$ for
  all $z'$ satisfying $(z',p) \models \text{\textsc{Post}}(\sigma)$.
  (Note that this is the same $p \in s(z,p^0)$ that appeared before,
  because there can only be one silent transition before an observable
  one and only silent transitions change the parameters.)  From
  Definition~\ref{hybrid-agent} we then have that $z
  \stackrel{\sigma[p]}{\hookrightarrow} z'$, and $(z',q') \in
  \mathfrak{R}$ because $V_M(z',p) = q'$.
\end{proof}

We have thus shown that whatever sequence of labels is observed in a
run of $A(H_a)$, a succession of continuous component dynamics with this same
sequence of subscript indices
can be activated in $H_a$.  Thus, whatever strategy is devised in $A(H_a)$,
has a guaranteed implementation in the concrete dynamics of the hybrid
agent.  The issue of selecting the parameters so that the implementation
is realized is not treated here.  This subject is addressed, using slightly different
discrete models, in \cite{acc12}.

%========================================================

\section{Case Study}
\label{section:example}

\subsection{Experimental Setup}

To demonstrate the efficacy of our methodology, we consider a game,
played between a robot and an intelligent adversary.  The purpose of
the robot (hybrid agent) is to visit all four rooms in the triangular
``apartment'' configuration of Fig.~\ref{fig:game}.  The four rooms in
this triangular apartment are connected through six doors, which an
intelligent adversary can close almost at will, trying to prevent the
robot from achieving its goal. Table~\ref{tab:rules} shows three
possible rule regimes that the adversary could use. Initially the robot is
capable of distinguishing closed from open doors, but it does not know
which doors can be closed simultaneously. In fact, it assumes that
only the initially closed doors are ones that can be closed.

\begin{table}[h!]\normalsize
%\vskip+5pt
\centering
\begin{tabular}{ l  l }
  \toprule
  Rules & Description \\
  \midrule
  \texttt{Opposite} & Only one pair of doors opposite to each other can be closed at any time:
  \\
  & $\{a,d\},\, \{a,e\},\, \{a,f\},\, \{b,f\},\,\{c,e\},\, \{e,f\}$  \\
  \addlinespace[5pt]
  \texttt{Adjacent} & Only one pair of doors adjacent to each other can be closed at any time:
  \\
  & $\{a,b\},\, \{a,c\},\, \{b,c\},\, \{b,d\},\, \{b,e\},\, \{c,d\},\,
  \{c,f\},\, \{d,e\},\, \{d,f\}$ \\ \addlinespace[5pt]
  \texttt{General}  & Any pair of doors can be closed at any time.                             \\
  \bottomrule
\end{tabular}
\caption{Some possible rules for the adversary (controlling the doors):  at each round,
  the environment either keeps static or opens exactly one door in the
  closed pair of doors and closes exactly one, which results in
  another pair of closed doors.\label{tab:rules}}
\vspace{-5ex}
\end{table}

\begin{figure}[h!]
\centering
\subfigure[The triangle room game representation.]{
     \label{fig:sketch}
    \includegraphics[width=0.4\textwidth]{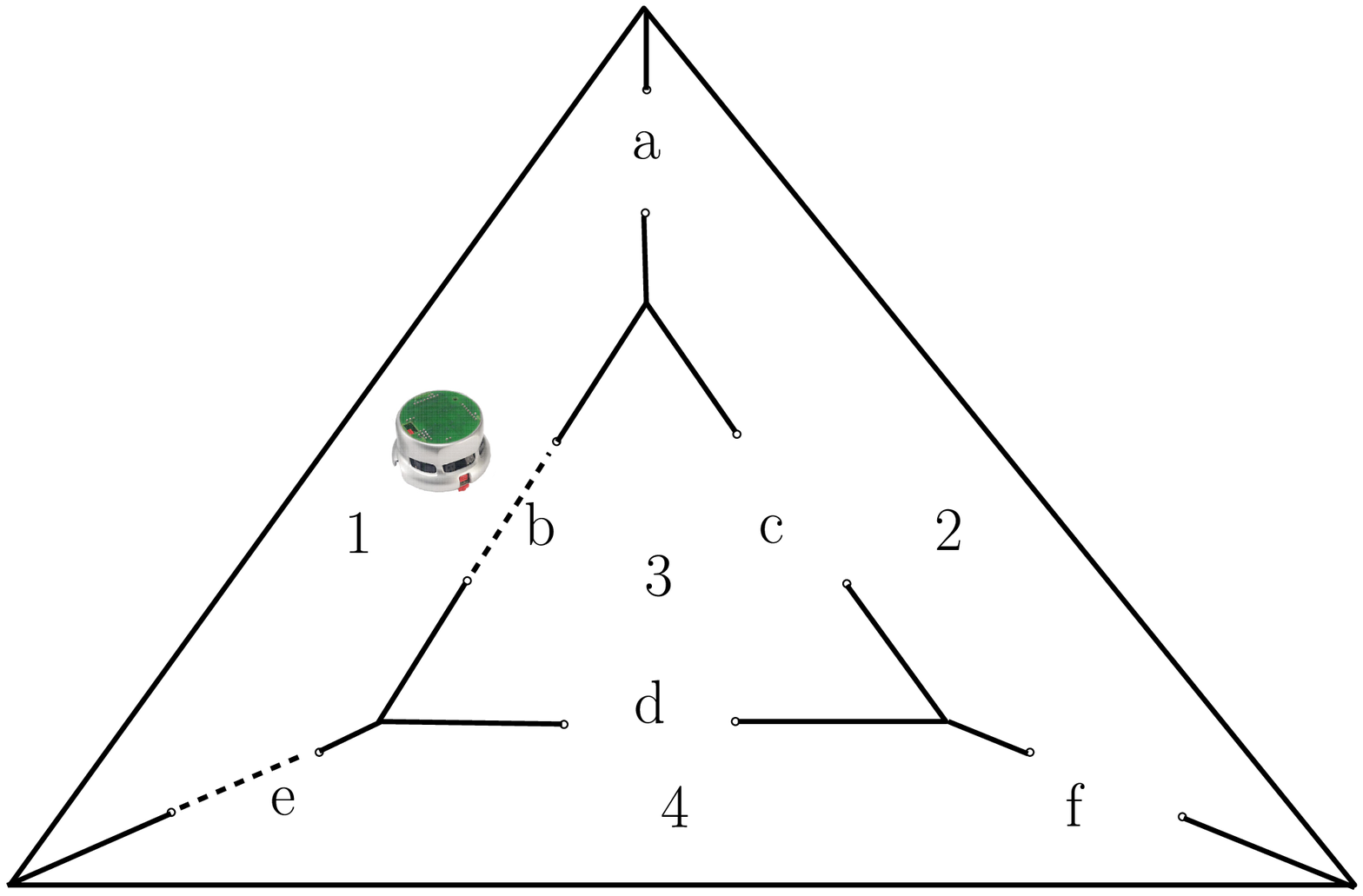}}
\subfigure[A physical implementation of the game.]{
     \label{fig:testbed}
     \includegraphics[width=0.4\textwidth]{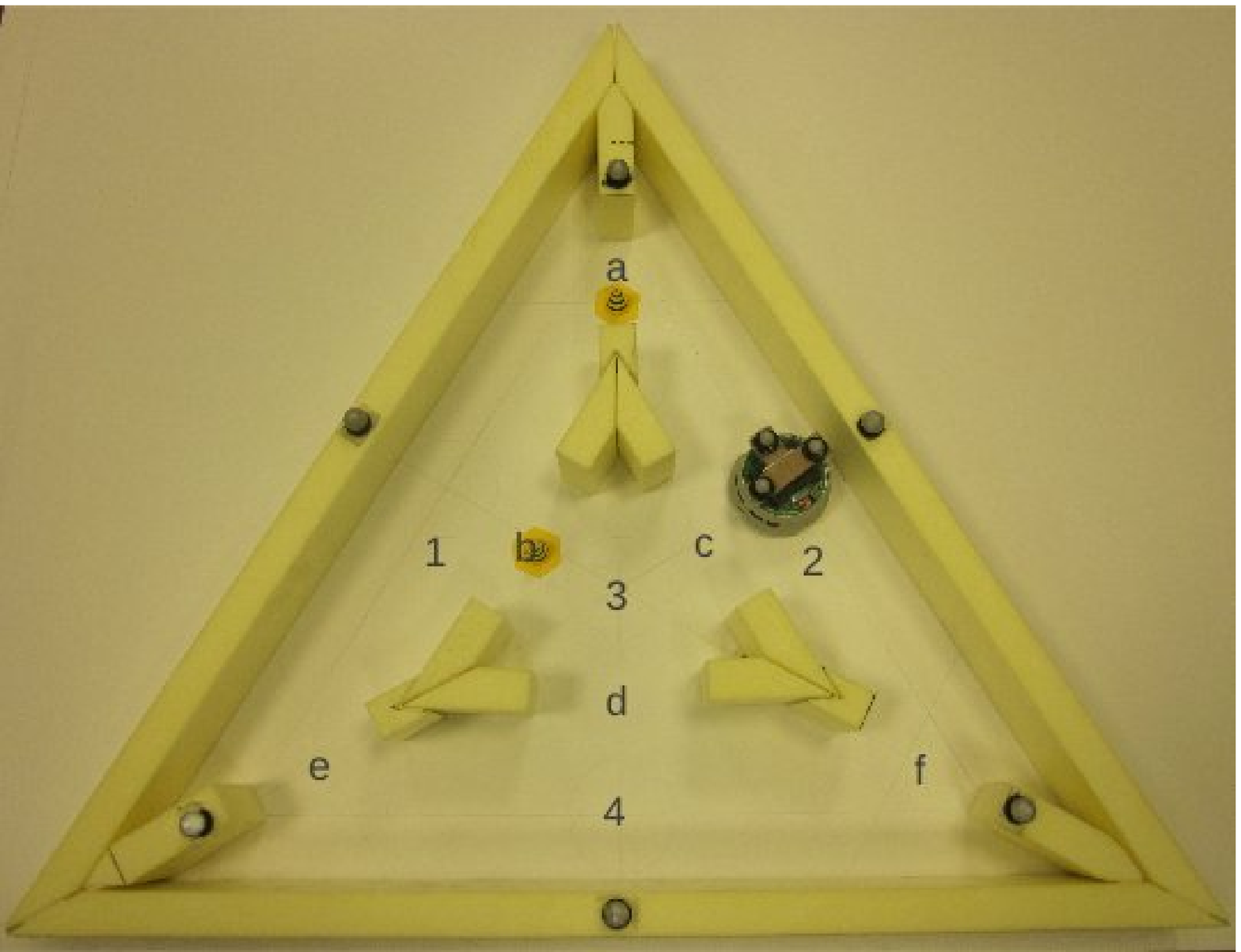}}
\hfill
\caption{The non-cooperative game used in this case study.
  Figure~\ref{fig:sketch} is a graphical depiction of the triangular
  apartment game, while Fig.~\ref{fig:testbed} shows a physical
  realization of the scenario, with a \texttt{Khepera II} miniature
  mobile robot in the role of the hybrid agent.  The robot localizes
  itself and observes which doors are closed (door closure implemented
  manually using the yellow caution cones) through a
  \textsc{vicon}\texttrademark~motion capture system.  The grammatical
  inference module and the strategy computation algorithm have been
  implemented in python, which communicates with the control for the
  robot (through Matlab{\texttrademark}) over a serial
  link.\label{fig:game}}
\vspace*{-5ex}
\end{figure}

The \texttt{Khepera II}, manufactured by K-Team Inc., is a
differential-drive mobile robot, with two actuated wheels and
kinematics that are accurately represented by the equations of a
unicycle. Motion control is achieved through \textsc{pid} loops that
independently control either angular displacement or speed of the two
wheels.  These \textsc{pid} loops can support the development of
mid-level motion planning controllers.  For example, input-output
feedback linearization of the unicycle dynamics \cite{vijay-korea}
leads to a fully actuated reduced system of the form $\dot{q} = u$,
where the sequential composition flow-through approach of
\cite{conner} can be applied to produce controllers that steer the
robot from room $i$ to a neighboring room $j$.  This same approach has
been used in \cite{waldo} to generate discrete abstractions for the
purpose of finding Waldo; details on how the sequential composition
approach can give rise to finite state automata abstractions are found
in \cite{conner-phd}.

For the case at hand, we can use the flow-through strategies 
to generate potential field-based velocity controllers to realize
transitions from room $i$ to room $j$ in a way compatible to the
requirements on the continuous dynamics of the hybrid agent of
Definition~\ref{hybrid-agent}, that is, ensure that $\text{\textsc{Pre}}(\sigma)$
is positively invariant for $f_\sigma$, 
and that trajectories converge to 
$L^+(p,\sigma) \oplus \mathcal{B}_\varepsilon$ in finite time (see \cite{conner-phd}).
The latter set is in fact the formula for $\text{\textsc{Post}}(\sigma)$:
$x \in L^+(p,\sigma) \oplus \mathcal{B}_\varepsilon$.

In the context of the flow-through navigation strategy of
\cite{conner}, a transition from, say, room 1 to room 2 (see
Fig.~\ref{fig:game}) would involve a \emph{flow-through vector field}
\cite{conner} by which the robot exits the polygon outlining room 1
from the edge corresponding to door $a$ (slightly more sophisticated
behavior can be produced by concatenating the flow-through policy with
a \emph{convergent} \cite{conner} one that ``centers'' the robot in
room 2.)

The hybrid agent that is obtained by equipping the robot with these
flow-through policies can be defined as a tuple
$H_a= \langle { \mathcal{Z}, \Sigma_a, \iota, \mathcal{P}, \pi_i, 
	\mathcal{AP}, f_\sigma, \text{\textsc{Pre}} ,\text{\textsc{Post}}, s, T_a} \rangle$ where
\begin{itemize}
\item $\mathcal{Z}$ is the triangular sector of $\mathbb{R}^2$ consisted of the
union of the areas of the four rooms.
\item $\Sigma_a = \{(1,2),\, (1,3), \,(1,4), \,(2,1), \,(2,3), \,(2,4), \,(3,1), \,(3,2), 
\,(3,4), \,(4,1), \,(4,2), \,(4,3)\}$, with each element associated with a single flow-through
policy: $(i,j)$ denotes a flow-through policy from room $i$ to room $j$.  
\item $\iota : \Sigma_a \to \{1,\, 2,\, 3,\, 4\}$ where we slightly abuse notation and
define $\iota$ not as a bijection but rather a surjection, where we abstract away the
room of origin and we maintain the destination, for simplicity.
\item $\pi_i = \pi = I$ (the identity), $\mathcal{P} = \mathcal{Z}$, and
$s(z,p) = \mathcal{P}$, $\forall (z,p) \in \mathcal{Z} \times \mathcal{P}$; 
in this case we do not have to use parameters
explicitly---they are hard-wired in the flow-through policies.
\item $\mathcal{AP}=$ $\{\alpha_i: \text{ robot in room } i\}$, $i=1,2,3,4$.
\item $f_\sigma = K (X_\sigma - \dot{q})$, $K >0$, a simple proportional controller on velocity
intended to align the system's vector field with the flow-through field $X_\sigma$.
\item $\text{\textsc{Pre}}\big( (i,\cdot) \big) = \alpha_i$, $i \in
  \{1,\ldots,4\}$ and $\text{\textsc{Post}}\big( (\cdot,j) \big) =
  \alpha_j$, $j \in \{1,\ldots,4\}$.
\item $T_a$ following Definition~\ref{hybrid-agent}, once all other components 
are defined.
\end{itemize}

One can verify by inspection when constructing $A(H_a)$, that the
first element of $\sigma = (i,j)$ is encoded in the label for the
discrete state, $\alpha_i$, from which the transition $\alpha_i
\stackrel{(i,j)}{\to} \alpha_j$.  Thus, to simplify notation, we
change the label of a state from $\alpha_i$ to $i$, and the label of
the transition from $(i,j)$ to just $j$---the destination state.  We
write $i \stackrel{j}{\to} j$ instead.  Figure~\ref{fig:SAs} (left)
gives a graphical representation of $A(H_a)$ after the
state/transition relabeling, basically expressing the fact that with
all doors open, the robot can move from any room to any other room by
initiating the appropriate flow-through policy.

\subsection{Results}

Suppose the adversarial environment adheres to the \texttt{Opposite}
 rule in Table~\ref{tab:rules}.  The
\ac{sa} $A_1$ for the agent (player 1) and a fragment of \ac{sa} $A_2$
modeling the environment (player 2) are shown in
Fig.~\ref{fig:SAs}.\footnote{SAs $A_1$ and $A_2$ happen to be Myhill graphs,
but the analysis presented applies to general \ac{sa}s.}
By assigning $I_1 = Q_1$ and $I_2=Q_2$, the game
can start with any state in $Q_1\times Q_2\times \{\bm 1\}$.

\begin{figure}[h]
  \begin{center}
    \begin{minipage}{.3\textwidth}
      \unitlength=4pt
      \begin{picture}(20, 20)(0,-2)
        \gasset{Nw=5,Nh=5,Nmr=2.5,ELside=r,curvedepth=0}
        \thinlines
        \node(A0)(0,13){$1$} 
        \node(A1)(13,13){$2$}
        \node(A2)(13,0){$3$} 
        \node(A3)(0,0){$4$} 
        \drawedge[ELside=l](A0,A1){$2$}
        \drawedge[curvedepth=-3](A1,A0){$1$}
        \drawedge[ELpos=25,curvedepth=-1](A2,A0){$1$}
        \drawedge[ELpos=25](A0,A2){$3$}
        \drawedge[curvedepth=-3](A0,A3){$4$}
        \drawedge[ELside=l](A3,A0){$1$}
        % \drawedge(A3,A0){$e$}
        \drawedge[ELside=l](A1,A2){$3$}
        \drawedge[curvedepth=-3](A2,A1){$2$}
        % \drawedge(A2,A1){$c$}
        \drawedge[ELpos=25,curvedepth=-1](A1,A3){$4$}
        \drawedge[ELpos=25](A3,A1){$2$}
        % \drawedge(A3,A1){$f$}
        \drawedge[ELside=l](A2,A3){$4$}
        \drawedge[curvedepth=-3](A3,A2){$3$}
        % \drawedge(A3,A2){$d$}
      \end{picture}
    \end{minipage}\hspace{1cm}
    \begin{minipage}{.5\textwidth}
     \unitlength=4pt
      \begin{picture}(50, 20)(0,0)
        \gasset{Nw=5,Nh=5,Nmr=2.5,curvedepth=0,loopdiam=2} \thinlines
        \node(A0)(0,8){$ad$} \node(A1)(20,8){$af$}
        \node(A2)(40,16){$bf$} \node(A3)(40,0){$ef$}
        \drawloop(A0){$\epsilon$} \drawloop(A1){$\epsilon$}
        \drawloop[loopangle=0](A2){$\epsilon$}
        \drawloop[loopangle=0](A3){$\epsilon$}
        \node[Nframe=n](B)(-10,20){$\ldots$}
        \node[Nframe=n](B0)(10,0){$\ldots$}
        \node[Nframe=n](B2)(50,10){$\ldots$}
        \node[Nframe=n](B3)(50,5){$\ldots$}
        \node[Nframe=n](B4)(10,20){$\ldots$} \drawedge(A0,A1){$af$}
        \drawedge[curvedepth=2](A1,A0){$ad$}
        \drawedge[curvedepth=-3](A1,A3){$ef$}
        \drawedge[curvedepth=4](A3,A1){$af$}
        \drawedge[curvedepth=4](A1,A2){$bf$}
        \drawedge[curvedepth=-3](A2,A1){$af$}
        \drawedge[curvedepth=2](A2,A3){$ef$} \drawedge(A3,A2){$bf$}
        \drawedge[dash={1.0}0](A0,B0){} \drawedge[dash={1.0}0](B,A0){}
        \drawedge[dash={1.0}0](A2,B2){}
        \drawedge[dash={1.0}0](A3,B3){}
        \drawedge[dash={1.0}0](B4,A1){}
      \end{picture}
    \end{minipage}
   \end{center}
   \caption{Semiautomata for the agent (left) and for a fragment of
     the environment (right).  In $A_1$, the states are the rooms and the transitions are labeled
with the rooms that the agent is to enter. For $A_2$, the states
represent the pairs of doors that are currently closed and a
transition $xy$ indicates the pair of doors $x,y$ are to be closed. \label{fig:SAs}}
\vspace*{-2ex}
\end{figure}
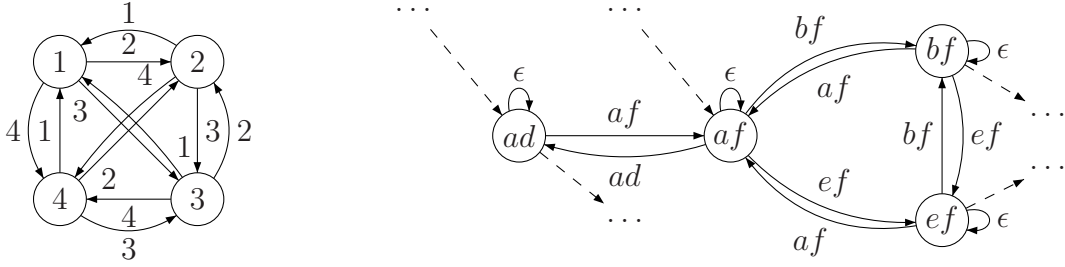

The goal of the agent in this example is to visit all four rooms (in
any order). Therefore, the specification can be described by the union of shuffle ideals of the permutations of $1234$. In this special case,
since the robot occupies one room when game starts, $\mathcal{A}_s=
\langle Q_s,\Sigma_s = \Sigma_1\cup \Sigma_2, T_s, I_s=\{1,2,3,4\},
F_s =\{1234\}\rangle$. A fragment of $\mathcal{A}_s$ is shown in Fig.~\ref{fig:taskfsa}.
\begin{figure}[h]
  \centering
  \begin{center}
    \unitlength=4pt
    \begin{picture}(60, 20)(10,0)
      \gasset{Nw=5,Nh=5,Nmr=2.5,loopdiam=3,curvedepth=0}
      \thinlines
  %    \node[Nmarks=i](A0)(0,10){0}
      \node[Nmarks=i](A1)(15,10){1}
      \node(A13)(30,0){13}
      \node(A12)(30,12){12}
      \node(A134)(45,0){134}
      \node(A123)(45,7){123}
      \node(A124)(50,17){124}
    %  \node[Nframe=n](B0)(5,0){$\ldots$}
      \node[Nframe=n](B1)(20,0){$\ldots$}
      \node[Nmarks=r,Nw=7](A1234)(60,7){1234}
    %  \drawedge(A0,A1){1}
      \drawedge(A1,A13){3}
      \drawedge(A1,A12){2}
      \drawedge[curvedepth=2](A12,A124){4}
      \drawedge(A12,A123){3}
      \drawedge(A13,A134){4}
      \drawedge(A13,A123){2}
      \drawedge(A124,A1234){3}
      \drawedge(A123,A1234){4}
      \drawedge(A134,A1234){2}
     % \drawloop(A0){$x$}
      \drawloop(A1){$x$,1}
      \drawloop(A12){$x$,1,2}
      \drawloop(A13){$x$,1,3}
      \drawloop(A123){$x$,1,2,3}
      \drawloop[loopangle=0](A134){$x$,1,3,4}
      \drawloop[loopangle=0](A124){$x$,1,2,4}
      \drawloop[loopangle=10](A1234){$x$,1,2,3,4}
    %  \drawedge[dash={1.0}0](A0,B0){}
      \drawedge[dash={1.0}0](A1,B1){}
    \end{picture}
  \end{center}
  \caption{Fragment of $\mathcal{A}_s = \langle Q_s,\Sigma_s =
    \Sigma_1\cup \Sigma_2, T_s, I_s=\{1,2,3,4\}, F_s
    =\{1234\}\rangle$, where $x =\Sigma_2$.}
  \label{fig:taskfsa}
\vspace*{-5ex}
\end{figure}
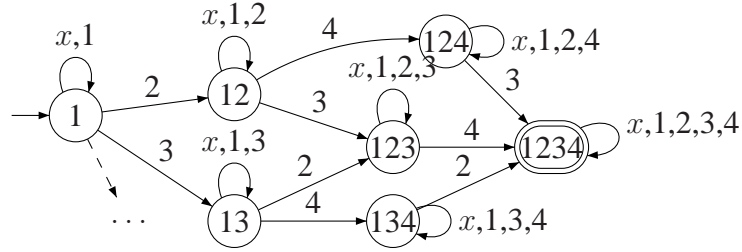

The interaction functions follow from obvious physical constraints: 
when the environment adversary closes a door, the agent cannot then move
through it.   The interaction function $U_2(d_1d_2,r)$ gives 
the set of rooms the agent cannot
access from room $r$ because doors $d_1$ and $d_2$ are closed.  
In Fig.~\ref{fig:testbed}, for instance, 
$U_2(ab,1)=\{2,3\}$.  In this example, the agent cannot enforce any constraints on
the adversary's behavior, so $U_1(q)=\emptyset,\forall q\in Q_1\times Q_2 $. 
Figure~\ref{fig:tbp} shows a fragment of $A_1\circ A_2$, while 
a fragment of the game automaton $\mathcal{G}$ is shown in Fig.~\ref{fig:gameboard}.

\begin{figure}[h]
  \centering
  \begin{center}
    \unitlength=4pt
    \begin{picture}(60, 25)(40,0)
      \gasset{Nw=12,Nh=4,Nmr=2.5,loopdiam=3,curvedepth=0}
      %\node(S1)(0,16){$(0,0,1)$}
      %\node(S2)(20,16){$(1,0,0)$}
      \node(A1)(40,16){$(1,ad,\bm1)$}
      \node(A2)(56,10){$(4,ad,\bm0)$}
      \node(A3)(56,20){$(3,ad,\bm0)$}
      \node(A4)(76,20){$(3,af,\bm1)$}
      \node(A5)(76,5){$(4,af,\bm1)$}
      % \node(A6)(92,32){$(1,af,0)$}
      \node(A7)(92,24){$(2,af,\bm0)$}
      \node(A8)(92,16){$(4,af,\bm0)$}
      \node(A9)(92,8){$(1,af,\bm0)$}
      \node(A10)(92,0){$(3,af,\bm0)$}
      %\node[Nframe=n,Nw=4,Nh=4](BS1)(5,8){$\ldots$}
      %\node[Nframe=n,Nw=4,Nh=4](BS2)(25,8){$\ldots$}
      \node[Nframe=n,Nw=4,Nh=4](B2)(70,0){$\ldots$}
      \node[Nframe=n,Nw=4,Nh=4](B3)(70,15){$\ldots$}
      % \node[Nframe=n,Nw=4,Nh=4](B6)(104,29){$\ldots$}
      \node[Nframe=n,Nw=4,Nh=4](B7)(104,21){$\ldots$}
      \node[Nframe=n,Nw=4,Nh=4](B8)(104,13){$\ldots$}
      \node[Nframe=n,Nw=4,Nh=4](B9)(104,5){$\ldots$}
      \node[Nframe=n,Nw=4,Nh=4](B10)(104,-3){$\ldots$}
      %%%%%%%%%%%%%%%%%%%%%%%%%%%%%%%%%%%%%%%%%%%%%%%%%%%%%%% 
      % \node[Nframe=n,Nw=4,Nh=4](D6)(104,35){$\ldots$}
      \node[Nframe=n,Nw=4,Nh=4](D7)(104,27){$\ldots$}
      \node[Nframe=n,Nw=4,Nh=4](D8)(104,19){$\ldots$}
      \node[Nframe=n,Nw=4,Nh=4](D9)(104,11){$\ldots$}
      \node[Nframe=n,Nw=4,Nh=4](D10)(104,3){$\ldots$}
      %%%%%%%%%%%%%%%%%%%%%%%%%%%%%%%%%%%%%%%%%%%%%%%%%% 
      \node[Nframe=n,Nw=4,Nh=4](C2)(65,15){$\ldots$}
      \node[Nframe=n,Nw=4,Nh=4](C3)(49,25){$\ldots$}
      \node[Nframe=n,Nw=4,Nh=4](C4)(69,25){$\ldots$}
      \node[Nframe=n,Nw=4,Nh=4](C5)(77,15){$\ldots$}
  %    \drawedge(S1,S2){$1$}
     % \drawedge(S2,A1){$ad$}
      \drawedge(A1,A2){$4$}
      \drawedge(A1,A3){$3$}
      \drawedge(A2,A5){$af$}
      \drawedge(A3,A4){$af$}
      \drawedge(A4,A9){$4$}
      \drawedge(A4,A7){$1$}
      \drawedge(A4,A8){$2$}
      \drawedge(A5,A9){$1$}
      \drawedge(A5,A10){$3$}
      %%%%%%%%%%%%%%%%%%%%%%%%%%%%%%%%%%%%%%%%%%%%% 
     % \drawedge[dash={1.0}0](S1,BS1){}
      %\drawedge[dash={1.0}0](S2,BS2){}
      \drawedge[dash={1.0}0](A2,B2){}
      \drawedge[dash={1.0}0](A3,B3){}
      % \drawedge[dash={1.0}0](A6,B6){}
      %%%%%%%%%%%%%%%%%%%%%%%%%%%%%%%%%%%%%%%%%%% 
      \drawedge[dash={1.0}0](C2,A2){}
      \drawedge[dash={1.0}0](C3,A3){}
      \drawedge[dash={1.0}0](C4,A4){}
      \drawedge[dash={1.0}0](C5,A5){}
      %%%%%%%%%%%%%%%%%%%%%%%%%%%%%%%%%%%%%%%%%%%%%%%%%%%%%%% 
      % \drawedge[dash={1.0}0](A6,B6){}
      \drawedge[dash={1.0}0](A7,B7){}
      \drawedge[dash={1.0}0](A8,B8){}
      \drawedge[dash={1.0}0](A9,B9){}
      \drawedge[dash={1.0}0](A10,B10){}
      % \drawedge[dash={1.0}0](D6,A6){}
      \drawedge[dash={1.0}0](D7,A7){}
      \drawedge[dash={1.0}0](D8,A8){}
      \drawedge[dash={1.0}0](D9,A9){}
      \drawedge[dash={1.0}0](D10,A10){}
    \end{picture}
  \end{center}
  \caption{Fragment of turn-based product $P=A_1\circ A_2 = \langle Q_p,\Sigma_1\cup\Sigma_2,T_p\rangle$. 
State $(r,d_1d_2,c)$ means the agent is in room $r$, doors
$\{d_1,d_2\}$ are closed and the Boolean variable keeping track of whose turn it
is set to $c$.
\label{fig:tbp}}
\end{figure}
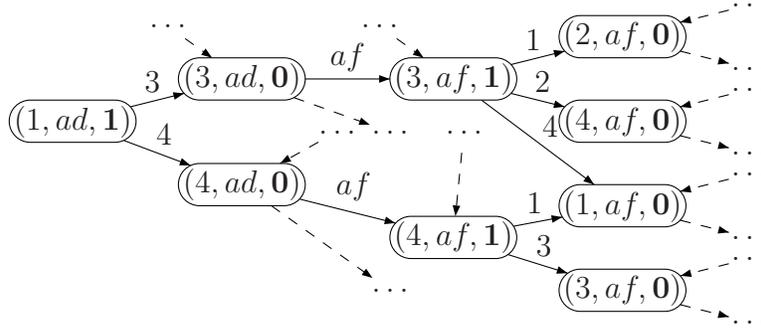

\begin{figure}[h!]
  \begin{center}
    \unitlength=4pt
    \vspace{-0.7in}
    \begin{picture}(60, 45)(0,0)
      \gasset{Nw=10,Nh=3,Nmr=2.5,loopdiam=3,curvedepth=0}
      %\node[Nmarks=i](S1)(-30,8){$(0,0,1),0$}
      %\node(S2)(-18,16){$(1,0,0),1$}
      \node[Nmarks=i](A1)(0,16){$\scriptstyle 1,ad,\bm{1},1$}
      \node(A2)(18,10){$\scriptstyle 4,ad,\bm{0},14$}
      \node(A3)(18,20){$\scriptstyle 3,ad,\bm{0},13$}
      \node(A4)(36,25){$\scriptstyle 3,af,\bm{1},13$}
      \node(A5)(36,5){$\scriptstyle 4,af,\bm{1},14$}
      \node(A6)(54,32){$\scriptstyle 1,af,\bm{0},13$}
      \node[Nw=17](A7)(54,24){$\scriptstyle 2,af,\bm{0},123$}
      \node[Nw=17](A8)(54,16){$\scriptstyle 4,af,\bm{0},124$}
      \node[Nw=17](A9)(54,8){$\scriptstyle 1,af,\bm{0},14$}
      \node[Nw=17](A10)(54,0){$\scriptstyle 3,af,\bm{0},134$}
     % \node[Nframe=n,Nw=4,Nh=4](BS1)(-20,0){$\ldots$}
      %\node[Nframe=n,Nw=4,Nh=4](BS2)(-5,8){$\ldots$}
      \node[Nframe=n,Nw=4,Nh=4](B2)(27,5){$\ldots$}
      \node[Nframe=n,Nw=4,Nh=4](B3)(27,15){$\ldots$}
      \node[Nframe=n,Nw=4,Nh=4](B6)(67,29){$\ldots$}
      \node[Nframe=n,Nw=4,Nh=4](B7)(67,21){$\ldots$}
      \node[Nframe=n,Nw=4,Nh=4](B8)(67,13){$\ldots$}
      \node[Nframe=n,Nw=4,Nh=4](B9)(67,5){$\ldots$}
      \node[Nframe=n,Nw=4,Nh=4](B10)(67,-3){$\ldots$}
      %%%%%%%%%%%%%%%%%%%%%%%%%%%%%%%%%%%%%%%%%%%%%%%%%%%%%%% 
      \node[Nframe=n,Nw=4,Nh=4](D6)(67,35){$\ldots$}
      \node[Nframe=n,Nw=4,Nh=4](D7)(67,27){$\ldots$}
      \node[Nframe=n,Nw=4,Nh=4](D8)(67,19){$\ldots$}
      \node[Nframe=n,Nw=4,Nh=4](D9)(67,11){$\ldots$}
      \node[Nframe=n,Nw=4,Nh=4](D10)(67,3){$\ldots$}
      %%%%%%%%%%%%%%%%%%%%%%%%%%%%%%%%%%%%%%%%%%%%%%%%%% 
      \node[Nframe=n,Nw=4,Nh=4](C2)(17,15){$\ldots$}
      \node[Nframe=n,Nw=4,Nh=4](C3)(17,25){$\ldots$}
      \node[Nframe=n,Nw=4,Nh=4](C4)(37,30){$\ldots$}
      \node[Nframe=n,Nw=4,Nh=4](C5)(37,10){$\ldots$}
     % \drawedge(S1,S2){$1$}
      %\drawedge(S2,A1){$ad$}
      \drawedge(A1,A2){$\scriptstyle 4$}
      \drawedge(A1,A3){$\scriptstyle 3$}
      \drawedge[ELpos=60](A2,A5){$\scriptstyle af$}
      \drawedge[ELpos=40](A3,A4){$\scriptstyle af$}
      \drawedge(A4,A6){$\scriptstyle 1$}
      \drawedge(A4,A7){$\scriptstyle 2$}
      \drawedge(A4,A8){$\scriptstyle 4$}
      \drawedge(A5,A9){$\scriptstyle 1$}
      \drawedge(A5,A10){$\scriptstyle 3$}
      %%%%%%%%%%%%%%%%%%%%%%%%%%%%%%%%%%%%%%%%%%%%% 
      %\drawedge[dash={1.0}0](S1,BS1){}
     % \drawedge[dash={1.0}0](S2,BS2){}
      \drawedge[dash={1.0}0](A2,B2){}
      \drawedge[dash={1.0}0](A3,B3){}
      \drawedge[dash={1.0}0](A6,B6){}
      %%%%%%%%%%%%%%%%%%%%%%%%%%%%%%%%%%%%%%%%%%% 
      \drawedge[dash={1.0}0](C2,A2){}
      \drawedge[dash={1.0}0](C3,A3){}
      \drawedge[dash={1.0}0](C4,A4){}
      \drawedge[dash={1.0}0](C5,A5){}
      %%%%%%%%%%%%%%%%%%%%%%%%%%%%%%%%%%%%%%%%%%%%%%%%%%%%%%% 
      \drawedge[dash={1.0}0](A6,B6){}
      \drawedge[dash={1.0}0](A7,B7){}
      \drawedge[dash={1.0}0](A8,B8){}
      \drawedge[dash={1.0}0](A9,B9){}
      \drawedge[dash={1.0}0](A10,B10){}
      \drawedge[dash={1.0}0](D6,A6){}
      \drawedge[dash={1.0}0](D7,A7){}
      \drawedge[dash={1.0}0](D8,A8){}
      \drawedge[dash={1.0}0](D9,A9){}
      \drawedge[dash={1.0}0](D10,A10){}
    \end{picture}
  \end{center}
  \caption{Fragment of the game automaton $\mathcal{G} =\langle Q,
    \Sigma_1\cup\Sigma_2, T,Q_0,F\rangle $ for the door-robot game,
    where $Q_0=\left\{(q_1,q_2,\bm 1, q_s)\mid q_1\in I_1, q_2\in I_2,
      q_s=q_1\in \{1,2,3,4\}\right\}$ and $F =\left\{(q_1,q_2,\bm 0,
      1234)\mid (q_1,q_2,\bm 0)\in Q_p\right\}$, note that upon
    initialization of a game, the state of $A_1$ (the room occupied by
    the robot) determines the choice of initial state in
    $\mathcal{A}_s$ (the room visited by the robot.)
    \label{fig:gameboard}}
\vspace*{-5ex}
\end{figure}
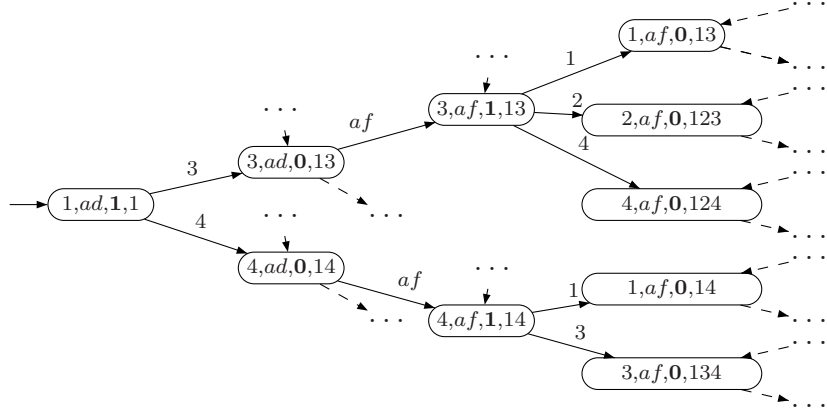

Let us show how Proposition~\ref{prop:optimal-moves} applies to this
case study.  The winning set of states is $F= \{\big((q_1,q_2,\bm
0),1234\big)\in Q\mid (q_1,q_2,\bm 0)\in Q_p \}$; $\mathsf{Attr}(F)$ is
obtained by computing the fixed-point of \eqref{eq:attractor}. Due to
space limitations, we  only give a winning path for the robot
according to the winning strategy $\mathsf{WS}^*_1$ with the initial
setting of the game in $Q_0$.

If the agent were to have complete knowledge of the game
automaton, it could compute the set of initial states from which
it has a winning strategy:
\[
Q_0 \cap
\mathsf{Attr}(F)= \big\{ (1,ad,\bm{1},1), (1,ce,\bm{1},1),
(2,ad,\bm{1},2), (2,bf,\bm{1},2), (4,ce,\bm{1},4),
(4,bf,\bm{1},4) \big\}. 
\]  

Hence, with complete game information, the robot can win the game
starting from initial conditions in $Q_0\cap \mathsf{Attr}(F)$; note that
$\frac{|Q_0\cap \mathsf{Attr}(F)|}{|Q_0|}$ makes up a mere
$25\%$ of all possible initial configurations. For instance, the
agent has no winning strategy if it starts in room
$3$.\footnote{Although the construction assumes the first move of the
  robot is to select a room to occupy (because it begins in state 0),
  we assume the game begins after the robot has been placed and the
  closed doors have been selected.}

For the sake of argument, take $q_0 = (1,ad,\bm{1},1) \in \mathsf{Attr}(F)
\cap Q_0$. Since the rank of $q_0$ is $\rho(q_0)=7$,  following $\mathsf{WS}^*_1$ of \eqref{eq:shortestws} the robot's fastest winning play is
\begin{multline*}
(1,ad,\bm{1},1)
\stackrel{4}{\rightarrow}
(4,ad,\bm{0},14)
\stackrel{ae}{\rightarrow}
(4,ae,\bm{1},14) 
\stackrel{2}{\rightarrow}
(2,ae,\bm{0},124)
\stackrel{ce}{\rightarrow}\\
(2,ce,\bm{1},124)
\stackrel{1}{\rightarrow}
(1,ce,\bm{0},124)
\stackrel{ef}{\rightarrow}
(1,ef,\bm{1},124)
\stackrel{3}{\rightarrow} 
(3,ef,\bm{0},1234)
\enspace.
\end{multline*}
The adversary's moves, $ae$, $ce$ and $ef$, are selected such that it
can slow down the process of winning of the robot as much as possible;
there is no move the environment can make to prevent the agent from
winning since the initial state is in the agent's attractor and the
agent has full knowledge of the game.  Note that in the cases where
the game rules are described by \texttt{Adjacent} and \texttt{General}
regimes (see Table~\ref{tab:rules}), the robot cannot win no matter
which initial state is in because in both cases $\mathsf{Attr}(F) \cap
Q_0 = \emptyset$. In these game automata, the agent, even with perfect
knowledge of the behavior of the environment, can never win.

Let us show how a robot, which has no prior knowledge of the game rules
but is equipped with a \ac{gim}, can start winning the game after a
point when it has observed enough to construct a correct
model of its environment.  As the first game starts, the agent
realizes that the environment is not static, but is rather expressed
by some (discrete) dynamical system, a \ac{sa} $A_2$.  It assumes
(rightfully so in this case) that the language admissible in $A_2$ is
strictly 2-local. With these knowledge, the robot's initial hypothesis
of the environment $A_2^{(0)} = \left(\langle Q_2,\Sigma_2,T_2\rangle,
  \mbox{sw}^{(0)}\right)$ is formulated in two
steps: \begin{inparaenum}[(i)]
\item obtain the $\text{\ac{sl}}_2$-\ac{fsa} for $
 \left\{\Sigma_2\setminus \{\epsilon\}\right\}^\ast$ and assign
  $\mbox{sw}^{(0)}(q,\sigma)= 1,\forall \sigma \in \Sigma_2\setminus
  \{\epsilon\}$;
\item add self-loops
  $T_2(q,\epsilon) = q$ and let $\mathrm{sw}^{(0)}(q,\epsilon)=1$ , $\forall q \in
  Q_2$.
\end{inparaenum}

In every round, the agent does the best it can: it takes the action suggested by the strategy
$\mathsf{WS}_1^\ast$ constructed based on its its current theory of mind. 
Each time it observes a new action on the part of its adversary, it 
updates its theory of mind using \eqref{eq:updatesw}, recomputes 
$\mathsf{WS}_1^\ast$ using \eqref{eq:shortestws}, 
and applies the new strategy in the following round.  The agent may realize that 
it has lost the game if it finds its current state out of the attractor computed based
on its most recent theory of mind.  In this case, the agent resigns and starts a new
game from a random initial condition, keeping the model for the environment it
has built so far and improving it as it goes.
We set an upper limit to the number of games by
restricting the total number of turns played to be less than $n$.

The following simplified algorithm illustrates
the procedure.  
\begin{enumerate}
\item \label{initial} Let $i=0$, the game hypothesis is
$\mathcal{G}^{(0)}$.  The game starts with a random 
  $q_0 \in Q_0$.
\item \label{computemove} At the current state $q
  =(q_1,q_2,\bm{1},q_s)$, if the number of turns exceeds the upper
  limit $n$, the sequence of repeated games is terminated. Otherwise,
  % if $q\in F$, the game is restarted at some $q_0\in Q_0$.
  the robot computes $\mathsf{Attr}(F)$ based on $\mathcal{G}^{(i)}$
  (note that it is not necessary to compute $\mathsf{Attr}(F)$ and
  $\mathsf{WS}_1^\ast(q)$ as long as there is no update in
  $\mathcal{G}^{(i)}$ from the previous round.)  Then, according to
  $\mathsf{Attr}(F)$ and \eqref{eq:termination}, the robot either
  makes a move $\sigma \in \mathsf{WS}_1^\ast(q)$ or resigns.  If a
  move is made and $T(q,\sigma)\in F$, the robot wins. In the case of
  either winning or resigning the game, the robot restarts the game at
  some $q_0\in Q_0$ with a theory of mind $A_2^{(i)}$ and a
  hypothesized game automaton $\mathcal{G}^{(i)}$; then its control
  goes to Step~\ref{computemove}. Otherwise, it goes to Step
  \ref{update}.
\item \label{update} The adversary takes some action. The robot observes this action
  and determines whether to switch on a blocked transition. If a new
  transition in $A_2^{(i)}$ is observed, it updates $A_2^{(i)}$ to
  $A_2^{(i+1)}$.  Then $\mathcal{G}^{(i)}$ is updated to
  $\mathcal{G}^{(i+1)}$ according to \eqref{eq:extendsw}. Otherwise,
  $A_2^{(i+1)}= A_2^{(i)}$ and
  $\mathcal{G}^{(i+1)}=\mathcal{G}^{(i)}$.  The robot sets $i=i+1$
  and goes to Step \ref{computemove}.
\end{enumerate}

We can measure the efficiency of the learning algorithm by computing 
the ratio between transitions that are switched on during the game sequence versus
the total number of enabled transitions in the true game automaton.
The convergence of learning is shown in 
Fig.~\ref{fig:convergeresult} and the results show that after 125 turns
including both robot's and environment's turns (approximately 42
games), the robot's model of the environment converges to the actual one.
%36 18
\begin{figure}[h!]
  \centering
\subfigure[\small The convergence of learning algorithm.  The 
	figure shows the ratio of adversary transitions that
       have been identified by the agent versus 
       the number of turns the two players have played.  In just
       125 turns the hybrid agent has full
       knowledge of its adversary's
       dynamics.]{
      \label{fig:convergeresult}
    \includegraphics[width=3in]{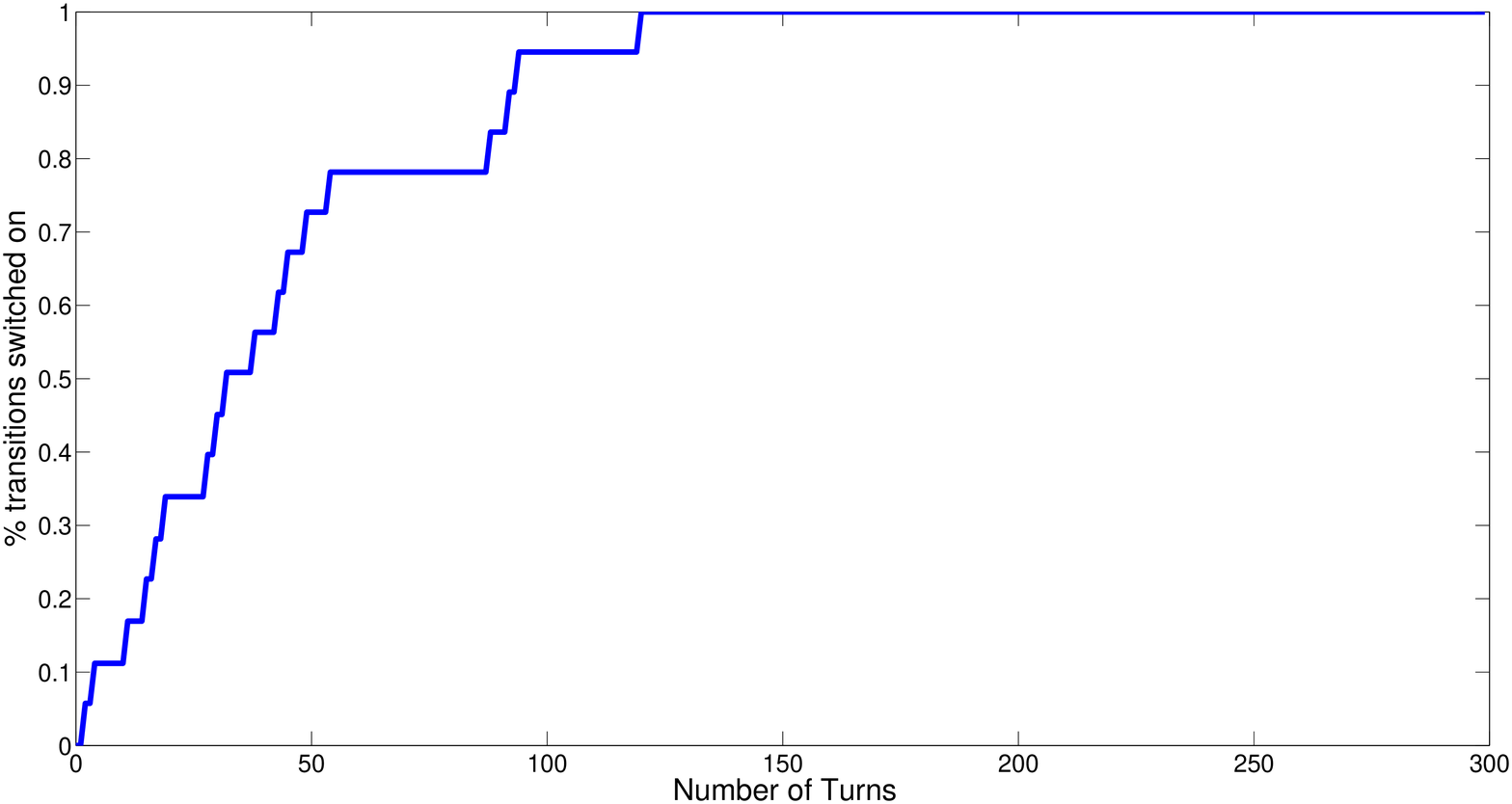}
    }
\subfigure[\small Comparison results with three
      types of the robot. For the case of ``no learning,'' the robot eventually 
     moves out of its attractor and gets
     trapped.]{
     \label{tbl:comparison}
 \begin{tabular}{ l  c c  c c }
 \toprule
   {} & \small Num of games & \small Num of wins \\
   \midrule
\small    No learning & \small 300 & \small 0 \\
\small    With learning  & \small 300 & \small 79 \\
\small    Full knowledge & \small 300 & \small 82\\
    \bottomrule
    \end{tabular}
}
\vspace*{-5ex}
\end{figure}

Table~\ref{tbl:comparison} gives outcomes of repeated games in three
different scenarios for the robot: \begin{inparaenum}[(a)] \item
  Full-Knowledge: the robot knows exactly the model of the
  environment; \item No Learning: the robot has no knowledge of, and
  no way of identifying the environment dynamics, and \item Learning:
  the robot starts without prior knowledge of environment dynamics but
  utilizes a \ac{gim}.
  \end{inparaenum}
The initial conditions for the game are chosen randomly.
In the absence of prior information about the environment
dynamics, and without any process for identifying it, the robot cannot win: 
in 300 games, it scores no victories.  If it had full knowledge
of this dynamics, it would have been able to win 82 out of the 300
times it played the game, a percentage of $27\%$, which is close to
the theoretical value of $25\%$.  A robot starting with no prior
knowledge but uses its \ac{gim} performs just as well (reaching a win
ratio of $26\%$) as one with full knowledge.  In
fact, as Fig.~\ref{fig:convergeresult} suggests, the robot has
recovered the performance of an ``all-knowing'' agent in less than
15\% ($\frac{42}{300}$) of the number of games played repetitively
used in Table~\ref{tbl:comparison}. We demonstrate the planning and
control of the robot using KiKS %\cite{kiks} 
simulation environment in
Matlab{\texttrademark}.\footnote{A simulation video is available at
\url{http://research.me.udel.edu/~btanner/Project_figs/newgame.mp4}.}
%\verb*<http://research.me.udel.edu/~btanner/Project_figs/newgame.mp4<.

%=========================================================

\section{Discussion and Conclusions}
\label{section:conclusion}
%\label{section:discuss}
This paper shows how the use of grammatical inference in robotic
planning and control allows an agent to perform a task in an unknown
and adversarial environment.  Within a game-theoretic framework, it is
shown that an agent can start from an incomplete model of its
environment and iteratively update that model via a string extension
learner applied to the language of its adversary's turns in the game,
to ultimately converge on the correct model.  Its success is
guaranteed provided that the language being learned is in the class of
languages that can be inferred from a positive presentation and the
characteristic sample can be observed.  This method leads to more
effective planning, since the agent will win the game if it is
possible for it to do so.  Our primary contribution is thus a
demonstration of how grammatical inference and game theory can be
incorporated in symbolic planning and control of a class of hybrid systems
with convergent closed loop continuous dynamics.

%Our approach accomplishes two goals: a) it offers a game-theoretic
%framework for model predictive planning which meets the uncertainty
%challenge presented by unknown, dynamic, but rule-governed environment
%dynamics, and b) it demonstrates how the inclusion of a \ac{gim} works
%as advertised: it makes a difference in terms of the hybrid agent's
%successful outcomes when it is supposed to. Moreover, the theory
%presented is modular and flexible, in the sense that as long as the
%hypothesized abstract models of the environment, the hybrid agent, and
%the task specification are finite-state, the \ac{gim} is guaranteed to
%work in the context of the framework outlined.

The architecture (framework) we propose is universal and can be seen
as being composed of two distinct blocks: Control synthesis and
Learning. The contents of these blocks can vary according to the task
in consideration and the target model to be learned. The current task
is a reachability problem, and hence we utilize algorithms for
computing a winning strategy in reachability games to synthesize
symbolic controllers.  However, there is nothing inherent in the
architecture that prevents synthesis of the control using winning
strategies of other types of games, such as  B\"uchi
games\cite{Mazala2001,Chatterjee2006}. Similarly, as in this paper the
rules of the environment are encoded in strictly $k$-local grammar,
the learning module operates on string extension languages. However,
any language that is identifiable from positive presentation can be
considered.  The main difference compared to our learning module
and other machine learning methods---such as reinforcement learning
and Bayesian inference---is that we take advantage of prior knowledge
about the structure of the hypothesis space.  This assumption enables
the development of faster and more efficient learning algorithms.

\bibliographystyle{IEEEtran}      
\bibliography{refers}

\end{document}